\documentclass[runningheads,a4paper]{llncs}

\usepackage{hyperref}
\usepackage{amssymb}
\usepackage{graphicx}
\usepackage{amsmath}
\usepackage{verbatim}
\usepackage{amssymb}
\usepackage[sectionbib]{natbib}
\usepackage{url}
\usepackage{xr}
\usepackage{fullpage}
\usepackage{appendix}
\usepackage{color}

\makeatletter
\def\@seccntformat#1{\@ifundefined{#1@cntformat}%
   {\csname the#1\endcsname\quad}  % default
   {\csname #1@cntformat\endcsname}% enable individual control
}
\let\oldappendix\appendix %% save current definition of \appendix
\renewcommand\appendix{%
    \oldappendix
    \newcommand{\section@cntformat}{\appendixname~\thesection\quad}
}
\makeatother

\urldef{\mailsa}\path|{hamedkarim@gmail.com},{jnutini, schmidtm}@cs.ubc.ca|

\def\norm#1{\|#1\|}
\newcommand{\R}{{\rm I\!R}}

\newcommand{\argmin}[1]{\mathop{\hbox{argmin}}_{#1}~\!}

\newcommand{\sign}[1]{\hbox{sign}(#1)}

\newcommand{\half}{\frac 1 2}

\newcommand{\be}{\begin{equation}}
\newcommand{\ee}{\end{equation}}
\newcommand{\bea}{\begin{eqnarray}}
\newcommand{\eea}{\end{eqnarray}}
\newcommand{\ben}{\begin{equation*}}
\newcommand{\een}{\end{equation*}}
\newcommand{\bean}{\begin{eqnarray*}}
\newcommand{\eean}{\end{eqnarray*}}

\definecolor{red}{rgb}{1,0,0}

\begin{document}

\mainmatter  % start of an individual contribution

\title{Linear Convergence of Gradient and Proximal-Gradient Methods Under the Polyak-\L{}ojasiewicz Condition}

\titlerunning{Linear Convergence Under the Polyak-\L{}ojasiewicz Condition}

\author{Hamed Karimi\and Julie Nutini \and Mark Schmidt%
%\thanks{Please note that the LNCS Editorial assumes that all authors have used
%the western naming convention, with given names preceding surnames. This determines
%the structure of the names in the running heads and the author index.}%
}

\institute{Department of Computer Science, University of British Columbia\\
Vancouver, British Columbia, Canada\\
\mailsa\\
%\url{https://www.cs.ubc.ca}
}


%\toctitle{Lecture Notes in Computer Science}
%\tocauthor{Authors' Instructions}
\maketitle

\begin{abstract} %**% MAXIMUM 150 WORDS %**% --- Not for the arXiv version!!
In 1963, Polyak proposed a simple condition that is sufficient to show a global linear convergence rate for gradient descent. This condition is a special case of the \L{}ojasiewicz inequality proposed in the same year, and it does not require strong convexity (or even convexity). In this work, we show that this much-older Polyak-\L{}ojasiewicz (PL) inequality is actually weaker than the main conditions that have been explored to show linear convergence rates without strong convexity over the last 25 years. We also use the PL inequality to give new analyses of randomized and greedy coordinate descent methods, sign-based gradient descent methods, and stochastic gradient methods in the classic setting (with decreasing or constant step-sizes) as well as the variance-reduced setting. 
We further propose a generalization that applies to proximal-gradient methods for non-smooth optimization, leading to simple proofs of linear convergence of these methods.
Along the way, we give simple convergence results for a wide variety of problems in machine learning: least squares, logistic regression, boosting, resilient backpropagation, L1-regularization, support vector machines, stochastic dual coordinate ascent, and stochastic variance-reduced gradient methods.
%Along the way, we give new convergence results for a wide variety of problems in machine learning: least squares, logistic regression, boosting, L1-regularization, support vector machines, stochastic dual coordinate ascent, and stochastic variance-reduced gradient methods.
%\keywords{Gradient descent, coordinate descent, stochastic gradient, variance-reduction, boosting, support vector machines, L1-regularization}
\end{abstract}

\section{Introduction}\label{sect:intro}

Fitting most machine learning models involves solving some sort of optimization problem. Gradient descent, and variants of it like coordinate descent and stochastic gradient, are the workhorse tools used by the field to solve very large instances of these problems. In this work we consider the basic problem of minimizing a smooth function and the convergence rate of gradient descent methods. It is well-known that if $f$ is strongly-convex, then gradient descent achieves a global linear convergence rate for this problem~\citep{Nes04b}. However, many of the fundamental models in machine learning  like least squares and logistic regression yield objective functions that are convex but not strongly-convex. Further, if $f$ is only convex, then gradient descent only achieves a sub-linear rate.

This situation has motivated a variety of alternatives to strong convexity (SC) in the literature, in order to show that we can obtain linear convergence rates for problems like least squares and logistic regression. One of the oldest of these conditions is the \emph{error bounds} (EB) of~\citet{Luo}, 
but four other recently-considered conditions are \emph{essential strong convexity} (ESC)~\citep{Liu_esc}, \emph{weak strong convexity} (WSC)~\citep{Necoara2015}, the \emph{restricted secant inequality} (RSI)~\citep{zhang-yin}, and the \emph{quadratic growth} (QG) condition~\citep{anitescu2000}. Some of these conditions have different names in the special case of convex functions. For example, a convex function satisfying RSI is said to satisfy \emph{restricted strong convexity} (RSC)~\citep{zhang-yin}. Names describing convex functions satisfying QG include \emph{optimal strong convexity} (OSC)~\citep{Liu_osc}, \emph{semi-strong convexity} (SSC)~\citep{gong2014linear}, and (confusingly) WSC~\citep{Ma2015}. The proofs of linear convergence under all of these relaxations are typically not straightforward, and it is rarely discussed how these conditions relate to each other.

In this work, we consider a much older condition that we refer to as the Polyak-\L{}ojasiewicz (PL) inequality. This inequality was originally introduced by~\citet{polyak}, who showed that it is a sufficient condition for gradient descent to achieve a linear convergence rate. We describe it as the PL inequality because it is also a special case of the inequality introduced in the same year by~\citet{loj}. We review the PL inequality in the next section and how it leads to a trivial proof of the linear convergence rate of gradient descent. Next, in terms of showing a global linear convergence rate to the optimal solution, we show that the PL inequality is \emph{weaker} than all of the more recent conditions discussed in the previous paragraph. This suggests that we can replace the long and complicated proofs under any of the conditions above with simpler proofs based on the PL inequality. Subsequently, we show how this result implies gradient descent achieves linear rates for standard problems in machine learning like least squares and logistic regression that are not necessarily SC, and even for some non-convex problems (Section~\ref{sec:problems}). In Section~\ref{sec:huge-scale} we use the PL inequality to give new convergence rates for randomized and greedy coordinate descent (implying a new convergence rate for certain variants of boosting), sign-based gradient descent methods, and stochastic gradient methods in either the classical or variance-reduced setting. Next we turn to the problem of minimizing the sum of a smooth function and a simple non-smooth function. We propose a generalization of the PL inequality that allows us to show linear convergence rates for proximal-gradient methods without SC. In this setting, the new condition is equivalent to the well-known Kurdyka-\L{}ojasiewicz (KL) condition which has been used to show linear convergence of proximal-gradient methods for certain  problems like support vector machines and $\ell_1$-regularized least squares~\citep{Bolte2015}. But this new alternate generalization of the PL inequality leads to shorter and simpler proofs in these cases.

%, and leads to simple to a simple analysis showing linear convergence of methods for  support vector machines and $\ell_1$ least squares problems. It also implies that we obtain a linear convergence rate for $\ell_1$-regularized least squares problems, showing that the extra conditions typically assumed to derive linear converge rates in this setting are in fact not needed.

\section{Polyak-\L{}ojasiewicz Inequality} \label{sec:PL}

We first focus on the basic unconstrained optimization problem
\be
\label{eq:unc}
\argmin{x \in \mathbb{R}^d} f(x),
\ee
 %the function $f$ 
%for the proof of concept we assume $\hbox{dom} f = \mathbb{R}^n$, but later we generalize the results for constraint minimization problem. 
and we assume that the first derivative of $f$ is $L$-Lipschitz continuous. This means that
\be 
\label{lip_cond}
f(y) \leq f(x) + \langle \nabla f(x) , y -x\rangle + \frac{L }{ 2} || y - x ||^2,
\ee
for all $x$ and $y$. For twice-differentiable objectives this assumption means that the eigenvalues of $\nabla^2 f(x)$ are bounded above by some $L$, which is typically a reasonable assumption. We also assume  that the optimization problem has a non-empty solution set $\mathcal{X}^*$, and we use $f^*$ to denote the corresponding optimal function value. We will say that a function satisfies the PL inequality if the following holds for some $\mu > 0$,
\be
\label{pl_ineq}
\frac{1}{2}|| \nabla f(x) ||^2 \geq \mu (f(x) - f^*), \quad \forall~x.
\ee
This inequality simply requires that the gradient grows faster than a quadratic function as we move away from the optimal function value. Note that this inequality implies that every stationary point is a global minimum. But unlike SC, it does not imply that there is a unique solution.
Linear convergence of gradient descent under these assumptions was first proved by~\citet{polyak}. Below we give a simple proof of this result when using a step-size of $1/L$.

\begin{theorem}\label{th:lin_pl}
Consider problem~\eqref{eq:unc}, where $f$ has an $L$-Lipschitz continuous gradient (\ref{lip_cond}), a non-empty solution set $\mathcal{X}^*$, and satisfies the PL inequality (\ref{pl_ineq}). Then the gradient method  with a step-size of $1/L$, 
\be
\label{gd_update}
x_{k+1}  = x_k - \frac{1 }{ L} \nabla f(x_k),
\ee
has a global linear convergence rate,
\begin{equation*}
f(x_k) - f^* \leq \left( 1- \frac{\mu }{ L}\right)^k (f(x_0) - f^*).
\end{equation*}
\end{theorem}
\begin{proof}
By using update rule~\eqref{gd_update} in the Lipschitz inequality condition~\eqref{lip_cond} we have
\ben 
f(x_{k+1}) -f(x_k) \leq  - \frac{1 }{ 2L} ||\nabla f(x_k)||^2.
\een
Now by using the PL inequality (\ref{pl_ineq}) we get
\bean
f(x_{k+1}) -f(x_k) %& \leq & - \frac{1 }{ 2L} ||\nabla f(x_k)||^2 \nonumber \\
& \leq & - \frac{\mu }{ L} (f(x_k) - f^*).
\eean
Re-arranging and subtracting $f^*$ from both sides gives us $f(x_{k+1}) - f^* \leq \left( 1- \frac{\mu}{ L}\right) (f(x_k) - f^*)$. Applying this inequality recursively gives the result. \qed
\end{proof}
Note that the above result also holds if we use the optimal step-size at each iteration, because under this choice we have
\[
f(x_{k+1}) = \min_{\alpha}\{f(x_k - \alpha\nabla f(x_k))\} \leq f\left(x_k - \frac{1}{L}\nabla f(x_k)\right). 
\]
A beautiful aspect of this proof is its simplicity; in fact it is \emph{simpler} than the proof of the same fact under the usual SC assumption. It is certainly simpler than typical proofs which rely on the other conditions mentioned in Section~\ref{sect:intro}. Further, it is worth noting that the proof does \emph{not} assume convexity of $f$. Thus, this is one of the few general results we have for global linear convergence on non-convex problems. 
%Differentiating and solving for $x^*$, we obtain
%\[
%	\nabla f(x) = 2 A^TAx + b \quad \Rightarrow \quad 2(A^TA)x^* = -b.
%\]
%From convexity, we have
%\begin{align}\label{convex}
%	f(y) & \ge f(x) + \langle \nabla f(x), y - x \rangle \nonumber \\
%	\Rightarrow \quad \quad \quad \quad \quad \quad f(x^*) & \ge f(x) + \langle \nabla f(x), x^* - x \rangle \nonumber \\
%	\iff \quad  \langle \nabla f(x), x - x^* \rangle & \ge f(x) - f(x^*).
%\end{align}
%Using the derived $x^*$, consider 
%\begin{align*}
%	2(A^TA)(x - x^*)
%	&= 2A^TAx - 2A^TAx^* \\
%	&= 2A^TAx + b \\
%	&= \nabla f(x).
%%	&= x + \frac{1}{2} (A^TA)^{-1} b \\
%%	&= \frac{1}{2} \left ( 2(A^TA)^{-1}(A^TA)x + (A^TA)^{-1} b \right ) \\
%%	&= \frac{1}{2}(A^TA)^{-1} \left ( 2A^TAx + b \right ) \\
%%	&= \frac{1}{2} (A^TA)^{-1}\nabla f(x) \\
%%	&\le \frac{M}{2} \nabla f(x),
%\end{align*}
%Using that $A^TA \ge M$ for some $M \ge 0$, we get 
%\[
%	x - x^* \le \frac{1}{2M} \nabla f(x).
%\]
%Combining this with equation \eqref{convex}, we have
%\[
%	\frac{1}{2} \| \nabla f(x) \|^2 \ge M f(x) - f(x^*).
%\]
%Thus, for $\mu = M$, $f$ satisfies the PL-inequality. We note that if $\nabla^2 f(x) = A^TA \succeq m \mathbb{I}$ for some parameter $m>0$, then $f$ is strongly-convex. However, if $m = 0$, then $f$ is convex (possibly strictly convex), but not strongly convex.

\subsection{Relationships Between Conditions}
\label{sec:weak}

As mentioned in the Section~\ref{sect:intro}, several other assumptions have been explored over the last 25 years in order to show that gradient descent achieves a linear convergence rate. These typically assume that $f$ is convex, and lead to more complicated proofs than the one above. However, it is rarely discussed how the conditions relate to each other. Indeed, all of the relationships that have been explored have only been in the context of convex functions~\citep{Bolte2015,Liu_osc,Necoara2015,Zhang2015}. In Appendix~\ref{app:weak}, we give the precise definitions of all  conditions and also prove the result below giving relationships between the conditions.
\begin{theorem}\label{thm:2}
For a function $f$ with a Lipschitz-continuous gradient, the following implications hold:
\[
(SC) \rightarrow (ESC) \rightarrow (WSC) \rightarrow (RSI) \rightarrow (EB) \equiv (PL) \rightarrow (QG).
\]
If we further assume that $f$ is convex then we have 
\[
(RSI) \equiv (EB) \equiv (PL) \equiv (QG).
\]
%and we note that RSC is defined by RSI plus convexity while OSC is defined by QG plus convexity.
\end{theorem}
Note the equivalence between EB and PL is a special case of a more general result by~\citet[][Theorem~5]{Bolte2015}, while~\citet{zhang2016characterization} independently also recently gave the relationships between RSI, EB, PL, and QG.\footnote{~\citet{drusvyatskiy2016error} is a recent work discussing the relationships among many of these conditions for non-smooth functions.}
This result shows that QG is the weakest assumption among those considered. However, QG allows non-global local minima so it is not enough to guarantee that gradient descent finds a global minimizer. This means that, among those considered above, \emph{PL and the equivalent EB are the most general conditions} that allow linear convergence to a global minimizer. 
Note that in the convex case QG is called OSC or SSC, but the result above shows that in the convex case it is also equivalent to EB and PL (as well as RSI which is known as RSC in this case).
%Another relationship is that if $f$ satisfies both PL and QG then it satisfies EB, while Zhang~\cite{Zhang2015} argues that QG and EB are equivalent for convex functions. This result implies that the \emph{only} condition we need to consider in the convex setting is the PL inequality, while for non-convex settings we only need to consider PL or QG. But in non-convex settings, QG is not enough to guarantee that gradient descent finds a global minimizer, so in all settings the PL inequality is the most general condition that allows linear convergence to a global minimizer.
\begin{comment}
\begin{center}
\scriptsize
	$\begin{matrix} \text{strong} \\ \text{convexity} \end{matrix}$
	$\Rightarrow$ $\begin{matrix} \text{essential} \\ \text{strong} \\ \text{convexity} \end{matrix}$
	$\Rightarrow$\!\!\!\!\!\! $\underbrace{\begin{matrix} \text{optimal} \\ \text{strong} \\ \text{convexity} \end{matrix}
	\equiv \begin{matrix} \text{weak} \\ \text{strong} \\ \text{convexity} \end{matrix}}_{\text{quadratic growth property \citep{Zhang2015}}}$\!\!\!\!\!
	$\equiv$ $\begin{matrix} \text{error} \\ \text{bound} \\ \text{property} \end{matrix}$
	$\equiv$ $\begin{matrix} \text{restricted} \\ \text{strong} \\ \text{convexity} \end{matrix}$
	$\begin{array}{l}
	\Rightarrow \text{convexity} \smallskip \\ 
	\Rightarrow \text{PL inequality}
	\end{array}$
\end{center}
In \cite{Zhang2015}, they show the equivalence between weak strong convexity, the error bound property and restricted strong convexity (with or without the Lipschitz continuity assumption -- clarify...).
\end{comment}

\subsection{Invex and Non-Convex Functions}

While the PL inequality does not imply convexity of $f$, it does  imply the weaker condition of \emph{invexity}. A function is invex if it is differentiable and there exists a vector valued function $\eta$ such that for any $x$ and $y$ in $\R^n$, the following inequality holds
\[
f(y) \ge f(x) + \nabla f(x)^T\eta(x,y).
\]
We obtain convex functions as the special case where $\eta(x,y) = y - x$.

Invexity was first introduced by~\citet{hanson1981}, and has been used in the context of learning output kernels~\citep{dinuzzo2011}. \citet{craven1985} show that a smooth $f$ is invex if and only if every stationary point of $f$ is a global minimum. Since the PL inequality implies that all stationary points are global minimizers, functions satisfying the PL inequality must be invex. It is easy to see this by noting that at any stationary point $\bar{x}$ we have $\nabla f(\bar{x}) = 0$, so we have 
\[
	0 = \frac{1}{2} \| \nabla f(\bar{x}) \|^2 \ge \mu (f(x) - f^*) \ge 0,
\]
where the last inequality holds because $\mu > 0$ and $f(x) \geq f^*$ for all $x$. This implies that $f(\bar{x}) = f^*$ and thus any stationary point must be a global minimum.

Theorem~\ref{thm:2} shows that all of the previous conditions (except QG) imply invexity. The function $f(x) = x^2 + 3\sin^2(x)$ is an example of an invex but non-convex function satisfying the PL inequality (with $\mu = 1/32$). Thus, Theorem~\ref{th:lin_pl} implies gradient descent obtains a global linear convergence rate on this function. 

%The relationships between the six conditions become more complex if we do not assume convexity. First, the error bound property (1) and the PL inequality remain equivalent even for non-convex functions.
%On the other hand, restricted strong convexity (4) and essential strong convexity (5) are defined using convexity so these are stronger conditions. Finally, despite their names optimal strong convexity (2) and weak strong convexity (3) can hold for non-convex functions. However, conditions (2) and (2) do not imply invexity so for non-convex functions gradient descent may reach a sub-optimal solution under these conditions.

Unfortunately, many complicated models have non-optimal stationary points. For example, typical deep feed-forward neural networks have sub-optimal stationary points and are thus not invex. A classic way to analyze functions like this is to consider a \emph{global convergence phase} and a \emph{local convergence phase}. The global convergence phase is the time spent to get ``close" to a local minimum, and then once we are ``close" to a local minimum the local convergence phase characterizes the convergence rate of the method. Usually, the local convergence phase starts to apply once we are locally SC around the minimizer. But this means that the local convergence phase may be arbitrarily small: for example, for $f(x) = x^2 + 3\sin^2(x)$ the local convergence rate would not even apply over the interval $x \in [-1,1]$. If we instead defined the local convergence phase in terms of locally satisfying the PL inequality, then we see that it can be \emph{much} larger ($x \in \R$ for this example).

\subsection{Relevant Problems}
\label{sec:problems}

If $f$ is $\mu$-SC, then it also satisfies the PL inequality with the same $\mu$ (see Appendix~\ref{app:problems}). Further, by Theorem~\ref{thm:2}, $f$ satisfies the PL inequality if it satisfies any of ESC, WSC, RSI, or EB (while for convex $f$, QG is also sufficient). Although it is hard to precisely characterize the general class of functions for which the PL inequality is satisfied, we note one important special case below.

\textbf{Strongly-convex composed with linear}: This is the case where $f$ has the form $f(x) = g(Ax)$ for some $\sigma$-SC function $g$ and some matrix $A$.  In Appendix~\ref{app:problems}, we show that this class of functions satisfies the PL inequality, and we note that this form frequently arises in machine learning. For example, least squares problems have the form
\[
f(x) = \norm{Ax-b}^2,
\]
%(\cite{garber2015} show that minimizing the least-squares problem over a strongly-convex set satisfies the PL inequality.)
and by noting that $g(z) \triangleq \norm{z-b}^2$ is SC we see that least squares falls into this category. Indeed, this class includes all convex quadratic functions.
%
%To show that this class of functions satisfies the PL inequality, we can first use the SC of $g$ to get
%\ben
%f(y) \geq f(x) + \langle \nabla f(x) , y-x \rangle + \frac{\sigma }{ 2} || A(y-x)||^2.
%\een
%Using $x_p$ to denote the projection of $x$ onto the optimal solution set $\mathcal{X}^*$, we have
%\begin{align*}
%	f(x_{p}) 
%	& \geq  f(x) + \langle \nabla f(x) , x_{p}-x \rangle + \frac{\sigma }{ 2} || A(x_{p}-x)||^2 \\
%	& \geq f(x) + \langle \nabla f(x) , x_{p}-x \rangle + \frac{\sigma \theta(A) }{ 2} ||x_{p}-x||^2 \\
%  	& \geq f(x)  + \min_y \left[\langle \nabla f(x) ,y-x \rangle + \frac{\sigma \theta(A) }{ 2} ||y-x||^2 \right] \\
%  	& = f(x) - \frac{1 }{ 2\theta(A)\sigma}||\nabla f(x)||^2.
%\end{align*}
%In the second line we use that $\mathcal{X}^*$ is polyhedral, and use the theorem of Hoffman~\cite{hoffman} to obtain a bound in terms of $\theta(A)$ (the smallest non-zero singular value of $A$). This derivation implies that the PL inequality is satisfied with $\mu = \sigma\theta(A)$.

In the case of logistic regression we have
\[
f(x) = \sum_{i=1}^n \log(1 + \exp(b_ia_i^Tx)).
\]
This can be written in the form $g(Ax)$, where $g$ is strictly convex but not SC. In cases like this where $g$ is only strictly convex, the PL inequality will still be satisfied over any compact set. Thus, if the iterations of gradient descent remain bounded, the linear convergence result still applies. It is reasonable to assume that the iterates remain bounded when the set of solutions is finite, since each step must decrease the objective function. Thus, for practical purposes, we can relax the above condition to ``strictly-convex composed with linear'' and the PL inequality implies a linear convergence rate for logistic regression.

\section{Convergence of Huge-Scale Methods}
\label{sec:huge-scale}

In this section, we use the PL inequality to analyze several variants of two of the most widely-used techniques for handling large-scale machine learning problems: coordinate descent and stochastic gradient methods. In particular, the PL inequality yields very simple analyses of these methods that apply to more general classes of functions than previously analyzed. We also note that the PL inequality has recently been used by~\citet{garber2015} to analyze the Frank-Wolfe algorithm. Further, inspired by the resilient backpropagation (RPROP) algorithm of~\citet{redmiller92}, in Appendix~\ref{app:huge-scale} we also give a convergence rate analysis for a sign-based gradient descent method.

\subsection{Randomized Coordinate Descent}\label{sec:coo-des}

\citet{nestrov2012} shows that randomized coordinate descent achieves a faster convergence rate than gradient descent for problems where we have $d$ variables and it is $d$ times cheaper to update one coordinate than it is to compute the entire gradient. The expected linear convergence rates in this previous work rely on SC, but in this section we show that randomized coordinate descent achieves an expected linear convergence rate if we only assume that the PL inequality holds.

To analyze coordinate descent methods, we assume that the gradient is coordinate-wise Lipschitz continuous, meaning that for any $x$ and $y$ we have
\bea
\label{lip_coo}
 	f(x + \alpha e_i) \leq f(x) + \alpha \nabla_i f(x) + \frac{L }{ 2} \alpha^2, \quad  \forall \alpha \in \mathbb{R}, \quad \forall x \in \mathbb{R}^d,
\eea
for any coordinate $i$, and where $e_i$ is the $i$th unit vector. 
\begin{theorem}\label{th:lin_coo}
Consider problem~\eqref{eq:unc}, where $f$ has a coordinate-wise $L$-Lipschitz continuous gradient (\ref{lip_coo}), a non-empty solution set $\mathcal{X}^*$, and satisfies the PL inequality (\ref{pl_ineq}). Consider the coordinate descent method  with a step-size of $1/L$, 
\be
\label{coo_up}
 x_{k+1} = x_k - \frac{1 }{ L} \nabla_{i_k} f(x_k) e_{i_k}.
\ee
If we choose the variable to update $i_k$ uniformly at random, then the algorithm
has an expected linear convergence rate of
\ben
\mathbb{E}[ f(x_k) - f^*] \leq \left( 1 - \frac{\mu }{ dL}\right)^k[ f(x_0) - f^*].
\een
\end{theorem}

\begin{proof}
By using the update rule (\ref{coo_up}) in the Lipschitz condition (\ref{lip_coo}) we have
\ben
f(x_{k+1}) \leq f(x_k) - \frac{1 }{ 2L} | \nabla_{i_k} f(x_k) |^2.
\een
By taking the expectation of both sides with respect to $i_k$ we have
\bean
\mathbb{E} \left [ f(x_{k+1}) \right ] &\leq & f(x_k) - \frac{1 }{ 2L} \mathbb{E} \left [ | \nabla_{i_k} f(x_k) |^2 \right ]\nonumber \\
& = & f(x_k) - \frac{1 }{ 2L}  \sum_i \frac{1}{ d}| \nabla_{i} f(x_k) |^2 \nonumber \\
& = & f(x_k) - \frac{1 }{ 2dL} ||\nabla f(x_k) ||^2.
\eean
By using the PL inequality \eqref{pl_ineq} and subtracting $f^*$ from both sides, %$|| \nabla f(x_k) ||^2 \geq 2\mu (f(x_k) - f^*)$, 
we get
\ben
\mathbb{E}[ f(x_{k+1}) - f^*] \leq \left( 1 - \frac{\mu }{ dL}\right)[f(x_k) - f^*].
\een
Applying this recursively and using iterated expectations yields the result. \qed
\end{proof}
As before, instead of using $1/L$ we could perform exact coordinate optimization and the result would still hold. If we have a Lipschitz constant $L_i$ for each coordinate and sample proportional to the $L_i$ as suggested by~\citet{nestrov2012}, then the above argument (using a step-size of $1/L_{i_k}$) can be used to show that we obtain a faster rate of
\ben
\mathbb{E}[ f(x_{k}) - f^*] \leq \left( 1 - \frac{\mu }{ d\bar{L}}\right)^k[f(x_0) - f^*],
\een
where $\bar{L} = \frac{1}{d}\sum_{j=1}^d L_j$.

\subsection{Greedy Coordinate Descent}

\citet{julie} have recently  analyzed coordinate descent under the greedy Gauss-Southwell (GS) rule, and argued  that this rule may be suitable for problems with a large degree of sparsity. The GS rule chooses $i_k$ according to the rule $i_k = \hbox{argmax}_j |\nabla_j f(x_k)|$. Using the fact that
\[
	\max_i |\nabla_i f(x_k)| \geq \frac{1}{d} \sum_{i=1}^d|\nabla_i f(x_k)|,
\]	
it is straightforward to show that the GS rule satisfies the rate above for the randomized method. 

However,~\citet{julie} show that a faster convergence rate can be obtained for the GS rule by measuring SC in the $1$-norm. Since the PL inequality is defined on the dual (gradient) space, in order to derive an analogous result we could measure the PL inequality in the $\infty$-norm,
\[
\frac 1 2 \norm{\nabla f(x)}_\infty^2 \geq \mu_1(f(x) - f^*).
\]
Because of the equivalence between norms, this is not introducing any additional assumptions beyond that the PL inequality is satisfied. Further, if $f$ is $\mu_1$-SC in the $1$-norm, then it satisfies the PL inequality in the $\infty$-norm with the same constant $\mu_1$. By using that $|\nabla_{i_k}f(x_k)| = \norm{\nabla f(x_k)}_\infty$ when the GS rule is used, the above argument can be used to show that coordinate descent with the GS rule achieves a convergence rate of
\[
f(x_k) - f^* \leq \left(1 - \frac{\mu_1}{L}\right)^k[f(x_0) - f^*],
\]
when the function satisfies the PL inequality in the $\infty$-norm with a constant of $\mu_1$. By the equivalence between norms we have that $\mu/d \leq \mu_1$, so this is faster than the rate with random selection.

\citet{ratsch} show that we can view some variants of boosting algorithms as implementations of coordinate descent with the GS rule. They use the error bound property to argue that these methods achieve a linear convergence rate, but this property does not lead to an explicit rate. Our simple result above thus provides the first explicit convergence rate for these variants of boosting.

\subsection{Stochastic Gradient Methods}

Stochastic gradient (SG) methods apply to the general stochastic optimization problem
\begin{equation}
\label{eq:stoch}
\argmin{x \in \R^d} f(x) = \mathbb{E}[f_i(x)],
\end{equation}
where the expectation is taken with respect to $i$. These methods are typically used to optimize finite sums,
\be
\label{eq:finite}
f(x) = \frac{1 }{ n}\sum_i^n f_i(x).
\ee
Here, each $f_i$ typically represents the fit of a model on an individual training example. SG methods are suitable for cases where the number of training examples $n$ is so large that it is infeasible to compute the gradient of all $n$ examples more than a few times.

Stochastic gradient methods use the iteration
\be
\label{sgd_up}
x_{k+1} = x_k - \alpha_k  \nabla f_{i_k}(x_k),
\ee 
where $\alpha_k$ is the step size and $i_k$ is a sample from the distribution over $i$ so that $\mathbb{E}[\nabla f_{i_k}(x_k)] = \nabla f(x_k)$. Below, we analyze the convergence rate of stochastic gradient methods under standard assumptions on $f$, and under both a decreasing and a constant step-size scheme.

\begin{theorem}\label{thm:sg}
Consider problem~\eqref{eq:stoch}. Assume that each $f$ has an $L$-Lipschitz continuous gradient (\ref{lip_cond}), $f$ has a non-empty solution set $\mathcal{X}^*$, $f$ satisfies the PL inequality (\ref{pl_ineq}), and $\mathbb{E}[\norm{\nabla f_i(x_k)}^2] \leq C^2$ for all $x_k$ and  some $C$. If we use the SG algorithm~\eqref{sgd_up} with $\alpha_k = \frac{2k+1 }{ 2\mu(k+1)^2}$, then we get a convergence rate of 
\ben
\mathbb{E}[f(x_k) - f^*] \leq \frac{L C^2 }{ 2 k\mu^2}.
\een
If instead we use a constant $\alpha_k = \alpha < \frac{1}{2\mu}$, then we obtain a linear convergence rate up to a solution level that is proportional to $\alpha$,
\[
\mathbb{E}[f(x_k) - f^*] \leq (1-2\mu\alpha)^k[f(x_0) - f^*] + \frac{LC^2\alpha}{4\mu}.
\]
\end{theorem}
\begin{proof}
By using the update rule~\eqref{sgd_up} inside the Lipschitz condition~\eqref{lip_cond}, we have
\[
	 f(x_{k+1}) \leq f(x_k)  - \alpha_k \langle f'(x_k),   \nabla f_{i_k}(x_k)  \rangle + \frac{L \alpha_k^2 }{ 2} ||  \nabla f_{i_k}(x_k)||^2. 
\]
Taking the expectation of both sides with respect to $i_k$ we have
\begin{align*}
	\mathbb{E}[f(x_{k+1})] 
 	& \leq f(x_k)  - \alpha_k \langle \nabla f(x_k),  \mathbb{E} \left [ \nabla f_{i_k}(x_k) \right ] \rangle + \frac{L \alpha_k^2 }{ 2}\mathbb{E}[\norm{\nabla f_i(x_k)}^2]  \\
 	& \leq f(x_k)  - \alpha_k || f'(x_k)||^2 + \frac{L C^2 \alpha_k^2 }{ 2} \\
 	& \leq f(x_k) - 2\mu \alpha_k ( f(x_k) - f^*) +  \frac{L C^2 \alpha_k^2 }{ 2},
\end{align*}
where the second line uses that $\mathbb{E}[ \nabla f_{i_k}(x_k)] = \nabla f(x_k)$ and $\mathbb{E}[\norm{\nabla f_i(x_k)}^2] \leq C^2$, and the third line uses the PL inequality. Subtracting $f^*$ from both sides yields:
\bea
\label{sgd_ineq}
\mathbb{E}[f(x_{k+1}) - f^*] \leq (1 - 2\alpha_k \mu)[f(x_{k}) - f^*] + \frac{L C^2 \alpha_k^2 }{ 2}.
\eea
{\bf Decreasing step size}: With $\alpha_k = \frac{2k + 1 }{ 2\mu(k+1)^2}$ in \eqref{sgd_ineq} we obtain 
\bean
 & \mathbb{E}[f(x_{k+1}) - f^*] \leq \frac{k^2 }{ (k+1)^2}[f(x_{k}) - f^*]  + \frac{L C^2 (2k+1)^2  |}{ 8 \mu^2 (k+1)^4}.&\nonumber
\eean
Multiplying both sides by $(k+1)^2$ and letting $\delta_f(k) \equiv k^2 \mathbb{E}[f(x_{k}) - f^*]$ we get
\begin{align*}
\delta_f(k+1) &\leq \delta_f(k)  + \frac{L C^2 (2k+1)^2 }{ 8 \mu^2 (k+1)^2} \nonumber \\
%\delta_f(k+1) 
&\leq \delta_f(k)  + \frac{LC^2 }{ 2 \mu^2},
\end{align*}
where the second line follows from $\frac{2k+1 }{ k+1 } < 2$. Summing up this inequality from $k=0$ to $k$ and using the fact that $\delta_f(0) = 0$ we get 
\bean
 &\delta_f(k+1) \leq \delta_f(0) + \frac{L C^2}{ 2 \mu^2} \sum_{i=0}^k 1 \leq \frac{L C^2 (k+1)}{ 2 \mu^2} \nonumber \\
  \Rightarrow \quad & (k+1)^2 \mathbb{E}[f(x_{k+1}) - f^*] \leq \frac{L C^2 (k+1)}{ 2 \mu^2}
\eean
which gives the stated rate. \\
{\bf Constant step size}: Choosing $\alpha_k = \alpha$ for any $\alpha < 1/2\mu$ and applying~\eqref{sgd_ineq} recursively yields
\begin{align*}
\mathbb{E}[f(x_{k+1}) - f^*] 
& \leq  (1 - 2\alpha \mu)^k[f(x_{0}) - f^*] + \frac{L C^2 \alpha^2 }{ 2} \sum_{i=0}^k  (1 - 2\alpha \mu)^i \\
& \leq  (1 - 2\alpha \mu)^k[f(x_{0}) - f^*] + \frac{L C^2 \alpha^2 }{ 2} \sum_{i=0}^{\infty}  (1 - 2\alpha \mu)^i \\
& = (1 - 2\alpha \mu)^k[f(x_{0}) - f^*]  + \frac{L C^2 \alpha }{ 4 \mu},
\end{align*}
where the last line uses that $\alpha < 1/2\mu$ and the limit of the geometric series. \qed
\end{proof}

The $O(1/k)$ rate for a decreasing step size matches the convergence rate of stochastic gradient methods under SC~\citep{nemirovski2009robust}. It was recently shown using a non-trivial analysis that a stochastic Newton method could achieve an $O(1/k)$ rate for least squares problems~\citep{bachMoulines2013}, but our result above shows that the basic stochastic gradient method already achieves this property (although the constants are worse than for this Newton-like method). Further, our result does not rely on convexity. Note that if we are happy with a solution of fixed accuracy, then the result with a constant step-size is perhaps the more useful strategy in practice: it supports the often-used empirical strategy of using a constant size for a long time, then halving the step-size if the algorithm appears to have stalled (the above result indicates that halving the step-size will at least halve the sub-optimality).

\subsection{Finite Sum Methods}\label{subsec:svrg}

In the setting of~\eqref{eq:finite} where we are minimizing a \emph{finite} sums, it has recently been shown that there are methods that have the low iteration cost of stochastic gradient methods but that still have linear convergence rates for SC functions~\citep{roux2012stochastic}. While the first methods that achieved this remarkable property required a \emph{memory} of previous gradient values, the stochastic variance-reduced gradient (SVRG) method of~\citet{johnson2013stochastic} does not have this drawback.~\citet{gong2014linear} show that SVRG has a linear convergence rate without SC under the weaker assumption of QG plus convexity (where QG is equivalent to PL). We review how the analysis of
~\citet{johnson2013stochastic} can be easily modified to give a similar result in Appendix~\ref{app:svrg}. A related result appears in~\citet{garber2015b}, who assume that $f$ is SC but do not assume that the individual functions are convex. More recent analyses by~\citet{reddi2,reddi1} have considered these types of methods under the PL inequality without convexity assumptions.

\section{Proximal-Gradient Generalization}\label{sec:prox-pl}

A generalization of the PL inequality for non-smooth optimization is the KL inequality~\citep{kurdyka1998gradients,bolte2008characterization}.
The KL inequality has been used to analyze the convergence of the classic proximal-point algorithm~\citep{AB} as well as a variety of other optimization methods~\citep{attouch2013convergence}. In machine learning, a popular generalization of gradient descent is proximal-gradient methods.
\citet{Bolte2015} show that the proximal-gradient method has a linear convergence rate for  functions satisfying the KL inequality, while~\citet{Li2016} give a related result. The set of problems satisfying the KL inequality notably includes problems like support vector machines and $\ell_1$-regularized least squares, implying that the algorithm has a linear convergence rate for these problems. In this section we propose a different generalization of the PL inequality that leads to a simpler linear convergence rate analysis for the proximal-gradient method as well as its coordinate-wise variant.

% Add more about KL-inequality???

%Most of the interesting large scale optimization problems have regularization terms, such as $l_1$ or $l_2$ regularizations. The popular method to deal with kind of problems are based on proximal gradient method, which are useful for non-differentiable regularization terms.\\ 
%Attouch and Bolte \cite{AB}, inspired by works of Kurdyak and \L{}ojasiewicz, showed that KL inequality is sufficient for proximal-point algorithms to achieve linear convergence to \textit{set of critical points} of $f(x)$.
%\be
%x_{k+1} \in \hbox{argmin} \{ f(u) + \frac{1 }{ \lambda}|u - x_k| ^2 \} 
%\ee
 
Proximal-gradient methods apply to problems of the form
\be
\label{f+g}
\argmin{x \in \mathbb{R}^d} F(x) = f(x) + g(x),
\ee 
where $f$ is a differentiable function with an $L$-Lipschitz continuous gradient and $g$ is a simple but potentially non-smooth convex function. Typical examples of simple functions $g$ include a scaled $\ell_1$-norm of the parameter vectors, $g(x) = \lambda\norm{x}_1$, and indicator functions that are zero if $x$ lies in a simple convex set and are infinity otherwise.
%For machine learning problems it could be the regularization term and in general combination of both, which covers many practical problems. 
%Although we could apply proximal-point algorithms to this problem, they are typically not practical. However, proximal-gradient methods are well-suited to solving problems with this structure.
%For most of machine learning optimization problems the proximal-point algorithm, although has nice convergence property, is not practical and first choice method, for these class of problems proximal-gradient methods are more appropriate and have been largely used, and in the following we present a generalization PL inequality for this setting.
% Here we show for all the functions satisfying Proximal-PL inequality in general, and LASSO problem in particular, proximal gradient methods achieves linear convergence rate.
%Now we look at the generalization of PL inequality for (\ref{f+g}). Before going through the details of the derivation and original motivations let's present the final result and what we call as \textit{Proximal-PL} inequality.
%We define the function $F(x) = f(x) + g(x)$, for $f$ with Lipschitz continuous derivative and convex function $g$, we call $F$ satisfies the Proximal-PL inequality, if there exist $\mu > 0$ such that
In order to analyze proximal-gradient algorithms, a natural (though not particularly intuitive) generalization of the PL inequality is that there exists a $\mu > 0$ satisfying
\begin{equation}\label{prox-pl}
\frac 1 2\mathcal{D}_g(x,L) \geq   \mu (F(x) - F^*),
\end{equation}
where
\begin{equation}
\mathcal{D}_g(x,\alpha) \equiv -2\alpha \min_y \left[ \langle \nabla f(x) , y-x \rangle + \frac{\alpha}{2}||y-x||^2 + g(y) - g(x) \right].
\end{equation}
We call this the \emph{proximal-PL} inequality, and we note that if $g$ is constant (or linear) then it reduces to the standard PL inequality.  
Below we show that this inequality is sufficient for the proximal-gradient method to achieve a global linear convergence rate.

\begin{theorem}\label{th:prox-pl}
Consider problem~\eqref{f+g}, where $f$ has an $L$-Lipschitz continuous gradient (\ref{lip_cond}), $F$ has a non-empty solution set $\mathcal{X}^*$, $g$ is convex, and $F$ satisfies the proximal-PL inequality~\eqref{prox-pl}. Then the proximal-gradient method  with a step-size of $1/L$, 
\bea\label{prox-up}
x_{k+1} = \argmin{y} \left[ \langle \nabla f(x_k) , y-x_k \rangle +  \frac{L }{ 2}||y-x_k||^2 + g(y) - g(x_k) \right]
\eea
converges linearly to the optimal value $F^*$,
\[ 
	F(x_{k}) - F^* \leq \left(1 - \frac{\mu }{ L}\right)^k[F(x_{0}) - F^*].
\]
\end{theorem}
\begin{proof}

By using Lipschitz continuity of the gradient of $f$ we have
\begin{align*}
\label{prox-lip}
F(x_{k+1}) & = f(x_{k+1}) + g(x_k) + g(x_{k+1}) - g(x_k) \nonumber\\
 &\leq F(x_k) + \langle \nabla f(x_k) , x_{k+1}-x_k \rangle + \frac{L }{ 2}||x_{k+1}-x_k||^2 + g(x_{k+1}) - g(x_k) \nonumber \\
	& \leq F(x_k) - \frac{1 }{ 2L} \mathcal{D}_g(x_k,L)\\
	& \leq F(x_k) - \frac{\mu}{L}[F(x_k) - F^*],
\end{align*}
which uses the definition of $x_{k+1}$ and $\mathcal{D}_g$ followed by the proximal-PL inequality~\eqref{prox-pl}. This subsequently implies that 
\be
 F(x_{k+1}) - F^*  \leq  \left(1 - \frac{\mu }{ L}\right) [ F(x_{k}) - F^*],
\ee
which applied recursively gives the result. \qed
\end{proof}
%We note that the condition $\mu \leq L$ is implicit in the definition of the proximal-PL inequality, but this is not restrictive since we can simply set $\mu$ to a smaller value to satisfy this.

While other conditions have been proposed to show linear convergence rates of proximal-gradient methods without SC~\citep{kadk2014,Bolte2015,Zhang2015,Li2016}, their analyses tend to be more complicated than the above. Further, in Appendix~\ref{app:proximalPLequiv} we show that the proximal-PL condition is in fact equivalent to the KL condition, which itself is known to be equivalent to a proximal-gradient variant on the EB condition~\citep{Bolte2015}. Thus, the proximal-PL inequality includes the standard scenarios where existing conditions apply.

\subsection{Relevant Problems}\label{subsec:probs}

As with the PL inequality, we now list several important function classes that satisfy the proximal-PL inequality~\eqref{prox-pl}. We give proofs that these classes satisfy the inequality in Appendix~\ref{app:probs} and~\ref{app:proximalPLequiv}.\begin{enumerate}
\item The inequality is satisfied if $f$ satisfies the PL inequality and $g$ is constant. Thus, the above result generalizes Theorem~\ref{th:lin_pl}.
\item The inequality is satisfied if $f$ is SC. This is the usual assumption used to show a linear convergence rate for the proximal-gradient algorithm~\citep{schmidt2011inexact}, although we note that the above analysis is much simpler than standard arguments.
\item The inequality is satisfied if $f$ has the form $f(x) = h(Ax)$ for a SC function $h$ and a matrix $A$, while $g$ is an indicator function for a polyhedral set.
\item The inequality is satisfied if $F$ is convex and satisfies the QG property.
\item The inequality is satisfied if $F$ satisfies the proximal-EB condition or the KL inequality.
%In Appendix~\ref{app:LSL1} we show that L1-regularized least squares and the support vector machine dual both fall into this category, and we discuss these two notable cases further below.
%A notable special case is the dual problem associated with support vector machines, which Ma et al. show satisfies the QG property~\citep{Ma2015}. Another notable case is L1-regularized least squares problems, which we show in the appendix satisfies the QG property.
%\item The inequality is satisfied if $f$ has the form $f(x) = g(Ax)$ for a strongly-convex function $g$ and a matrix $A$. This includes $\ell_1$-regularized least squares problems as a special case, which we discuss further in the next section.
%\item  If $F = f +g$ has optimal strong convexity property, then $F$ satisfies the inequality. Liu, Wright \cite{Liu_osc} proved that if $g = id_X$ is an indicator function, then $F(x) = f(Ax) + id_X$ has optimal strong convexity, but here we give a proof that for any convex function $g$ then $F(x) = f(Ax) + g(x)$ satisfies the inequality, but does not necessarily have optimal strong convexity.
\end{enumerate}
By the equivalence shown in Appendix~\ref{app:proximalPLequiv}, the proximal-PL inequality also holds for other problems where a linear convergence rate has been show like
%We expect that it is possible to show the proximal-PL inequality holds in other cases where  proximal-gradient methods achieve a linear convergence rate like the case of 
group L1-regularization~\citep{tseng2010}, sparse group L1-regularization~\citep{zhang2013linear}, nuclear-norm regularization~\citep{houNuclear}, and other classes of functions~\citep{zhou2015unified,drusvyatskiy2016error}.

\subsection{Least Squares with L1-Regularization}\label{subsec:LSL1}

Perhaps the most interesting example of problem~\eqref{f+g} is the $\ell_1$-regularized least squares problem,
\[
\argmin{x \in \R^d} \frac{1}{2}\norm{Ax - b}^2 + \lambda \norm{x}_1,
\]
where $\lambda > 0$ is the regularization parameter. This problem has been studied extensively in machine learning, signal processing, and statistics. This problem structure seems well-suited to using proximal-gradient methods, but the first works analyzing proximal-gradient methods for this problem only showed sub-linear convergence rates~\citep{beck2009fast}. Subsequent works show that linear convergence rates can be achieved under additional assumptions. For example,~\citet{gu} prove that their algorithm achieves a linear convergence rate if $A$ satisfies a \emph{restricted isometry property} (RIP) and the solution is sufficiently sparse. \citet{xiao_zhang} also assume the RIP property and show linear convergence using a homotopy method that slowly decreases the value of $\lambda$. \citet{agarwal} give a linear convergence rate under a \emph{modified restricted strong convexity} and \emph{modified restricted smoothness} assumption. But these problems have also been shown to satisfy proximal variants of the KL and EB conditions~\citep{tseng2010,Bolte2015,Necoara2015b}, and~\citet{Bolte2015} in particular analyzes the proximal-gradient method under KL while giving explicit bounds on the constant. This means \emph{any} L1-regularized least squares problem also satisfies the proximal-PL inequality. Thus, Theorem~\ref{th:prox-pl} gives a simple proof of global linear convergence for these problems without making additional assumptions or making any modifications to the algorithm. 

\subsection{Proximal Coordinate Descent}\label{subsec:proxCD}

It is also possible to adapt our results on coordinate descent and proximal-gradient methods in order to give a linear convergence rate for coordinate-wise proximal-gradient methods for problem~\eqref{f+g}. To do this, we require the extra assumption that $g$ is a separable function. This means that $g(x) = \sum_i g_i(x_i)$ for a set of univariate functions $g_i$.
The update rule for the coordinate-wise proximal-gradient method is
\bea
\label{coo-prox-up}
x_{k+1} & = & \argmin{\alpha} \left [ \alpha \nabla_{i_k} f(x_k) + \frac{L }{ 2}\alpha^2 + g_{i_k}(x_{i_k} + \alpha) - g_{i_k}(x_{i_k}) \right ],
\eea
We state the convergence rate result below.
\begin{theorem}\label{thm:prox-cd}
Assume the setup of Theorem~\ref{th:prox-pl} and that $g$ is a separable function $g(x) = \sum_i g_i(x_i)$, where each $g_i$ is convex. Then the coordinate-wise proximal-gradient update rule (\ref{coo-prox-up}) achieves a convergence rate
\be
\mathbb{E} [F(x_{k}) - F^*] \leq \left( 1 - \frac{\mu }{ dL}\right)^k[F(x_0) - F^*],
\ee
when $i_k$ is selected uniformly at random.
\end{theorem}
The proof is given in Appendix~\ref{app:proxCD} and although it is more complicated than the proofs of Theorems~\ref{thm:sg} and~\ref{th:prox-pl}, it is arguably still simpler than existing proofs for proximal coordinate descent under SC~\citep{richtarikTakac}, KL~\citep{attouch2013convergence}, or QG~\citep{zhang2016characterization}. It is also possible to analyze stochastic proximal-gradient algorithms, and indeed~\citet{reddi3} use the proximal-PL inequality to analyze finite-sum methods in the proximal stochastic case. 

\subsection{Support Vector Machines}\label{subsec:SVMs}

Another important model problem that arises in machine learning is support vector machines,
\be\label{prime-svm}
	\argmin{x\in\R^d} \frac{\lambda }{ 2} x^T x + \sum_{i=1}^n \max(0, 1-b_i x^Ta_i).
\ee
where $(a_i, b_i)$ are the labelled training set with $a_i \in \mathbb{R}^d$ and $b_i \in \{-1, 1\}$. 
We often solve this problem by performing coordinate optimization on its Fenchel dual, which has the form
\bea
\min_{\bar{w}} f(\bar{w}) = \frac{1 }{ 2} \bar{w}^T M \bar{w}  - \sum \bar{w}_i, \quad \bar{w}_i \in [0,U],
\eea 
for a particular positive semi-definite matrix $M$ and constant $U$. 
%In the appendix we show that this problem satisfies the proximal-PL inequality and thus the result of the previous section applies. 
 This convex function satisfies the QG property and thus  Theorem~\ref{thm:prox-cd} implies that coordinate optimization achieves a linear convergence rate in terms of optimizing the dual objective. Further, note that~\citet{hush} show that we can obtain an  $\epsilon$-accurate solution to the primal problem with an  $O(\epsilon^2)$-accurate solution to the dual problem. Thus this result also implies we can obtain a linear convergence rate on the primal problem by showing that stochastic dual coordinate ascent has a linear convergence rate on the dual problem. Global linear convergence rates for SVMs have also been shown by others~\citep{tsengYun,wang14,Ma2015}, but again we note that these works lead to more complicated analyses. Although the constants in these convergence rate may be quite bad (depending on the smallest non-zero singular value of the Gram matrix), we note that the existing sublinear rates still apply in the early iterations while, as the algorithm begins to identify support vectors, the constants improve (depending on the smallest non-zero singular value of the block of the Gram matrix corresponding to the support vectors).

The result of the previous section is not only restricted to SVMs. Indeed, the result of the previous subsection implies a linear convergence rate for many $\ell_2$-regularized linear prediction problems, the framework considered in the stochastic dual coordinate ascent (SDCA) work of~\citet{shalev-shwartz}. While~\citet{shalev-shwartz} show that this is true when the primal is smooth, our result gives linear rates in many cases where the primal is non-smooth. 
 
 \section{Discussion}
 
% While the proximal-PL inequality applies to the interesting problem classes above, 
% we note that it does not hold for some problems where we know that the proximal-gradient achieves a linear convergence rate. For example, if we have the function $f(x) = x_1^2 + |x_1|^3 + |x_2|^3$ and we place the later two terms in $g$, then the proximal-gradient method has a linear convergence rate even though the function does not satisfy the proximal-PL inequality. It is possible that a more general condition could be derived that makes use of properties of $g$ in order to include cases like this.

We believe that this work provides a unifying and simplifying view of a variety of optimization and convergence rate issues in machine learning. Indeed, we have shown that many of the assumptions used to achieve linear convergence rates can be replaced by the PL inequality and its proximal generalization. While we have focused on sufficient conditions for linear convergence, another recent work has turned to the question of necessary conditions for convergence~\citep{zhang2016characterization}. Further, while we've focused on non-accelerated methods,~\citet{zhang2016characterization} has recently analyzed Nesterov's accelerated gradient method without strong convexity. We also note that, while we have focused on first-order methods,~\citet{nesterov2006cubic} have used the PL inequality to analyze a second-order Newton-style method with cubic regularization. They also consider a generalization of the inequality under the name \emph{gradient-dominated} functions.

Throughout the paper, we have pointed out how our analyses imply convergence rates for a variety of machine learning models and algorithms. Some of these were previously known, typically under stronger assumptions or with more complicated proofs, but many of these are novel. Note that we have not provided any experimental results in this work, since the main contributions of this work are showing that existing algorithms actually work better on standard problems than we previously thought.
We expect that going forward efficiency will no longer be decided by the issue of whether functions are SC, but rather by whether they satisfy a variant of the PL inequality.

\subsubsection*{Acknowledgments.} 

We would like to thank Simon LaCoste-Julien, Martin Tak{\'a}{\v{c}}, Ruoyu Sun, Hui Zhang, and Dmitriy Drusvyatskiy for valuable discussions. We would like to thank Ting Kei Pong and Zirui Zhou for pointing out an error in the first version of this paper, to Ting Kei Pong for discussions that lead to the addition of Appendix~\ref{app:proximalPLequiv}, to J\'{e}r\^{o}me Bolte for an informative discussion about the KL inequality and pointing us to related results that we had missed, to Liam Madden and Stephen Becker for pointing out an error (and the fix) in our ``PL implies QG'' proof, and to Boris Polyak for providing an English translation of his original work. This research was supported by the Natural Sciences and Engineering Research Council of Canada (NSERC RGPIN-06068-2015).  
Julie Nutini is funded by a UBC Four Year Doctoral Fellowship (4YF) and Hamed Karimi is support by a Mathematics of Information Technology and Complex Systems (MITACS) Elevate Fellowship.

\appendix
\renewcommand{\thesubsection}{\Alph{subsection}}

\section{Relationships Between Conditions}\label{app:weak}
%We prove the results of Theorem 2: for convex functions $f$, the PL inequality condition is weaker than the error bound property, optimal strong convexity, essential strong convexity, restricted strong convexity and weak strong convexity.

We start by stating the different conditions. All of these definitions involve some constant $\mu > 0$ (which may not be the same across conditions), and we'll use the convention that $x_p$ is the projection of $x$ onto the solution set $\mathcal{X}^*$.

\begin{enumerate}
\item \textbf{Strong Convexity} (SC): For all $x$ and $y$ we have
\[
f(y) \geq f(x) + \langle \nabla f(x),y-x\rangle + \frac \mu 2 \norm{y-x}^2.
\]
\item \textbf{Essential Strong Convexity} (ESC): For all $x$ and $y$ such that $x_p = y_p$ we have
\[
f(y) \geq f(x) + \langle \nabla f(x),y-x\rangle + \frac \mu 2 \norm{y-x}^2.
\]
\item \textbf{Weak Strong Convexity} (WSC): For all $x$ we have
\[
f^* \geq f(x) + \langle \nabla f(x),x_p-x\rangle + \frac \mu 2 \norm{x_p-x}^2.
\]
\item \textbf{Restricted Secant Inequality} (RSI): For all $x$ we have
\[
\langle \nabla f(x),x-x_p\rangle \geq \mu \norm{x_p-x}^2.
\]
If the function $f$ is also convex it is called \textbf{restricted strong convexity} (RSC).
\item \textbf{Error Bound} (EB): For all $x$ we have
\[
\norm{\nabla f(x)} \geq \mu\norm{x_p -x}.
\]
\item \textbf{Polyak-\L{}ojasiewicz} (PL): For all $x$ we have
\[
\half \norm{\nabla f(x)}^2 \geq \mu(f(x) - f^*).
\]
\item \textbf{Quadratic Growth} (QG): For all $x$ we have
\[
f(x) - f^* \geq \frac \mu 2 \norm{x_p-x}^2.
\]
If the function $f$ is also convex it is called \textbf{optimal strong convexity} (OSC) or \textbf{semi-strong convexity} or sometimes WSC (but we'll reserve the expression WSC for the definition above).
\end{enumerate}
Below we prove a subset of the implications in Theorem~\ref{thm:2}. The remaining relationships in Theorem~\ref{thm:2} follow from these results and transitivity.
\begin{itemize}
\item $\bf SC \rightarrow ESC$: The SC assumption implies that the ESC inequality is satisfied for all $x$ and $y$, so it is also satisfied under the constraint $x_p = y_p$.
\item $\bf ESC \rightarrow WSC$: Take $y = x_p$ in the ESC inequality (which clearly has the same projection as $x$) to get WSC with the same $\mu$ as a special case.
\item $\bf WSC \rightarrow RSI$: Re-arrange the WSC inequality to
\[
\langle \nabla f(x),x-x_p\rangle \geq f(x) - f^* + \frac \mu 2 \norm{x_p-x}^2.
\]
Since $f(x) - f^* \geq 0$, we have RSI with $\frac \mu 2$.
\item $\bf RSI \rightarrow EB$: Using Cauchy-Schwartz on the RSI we have
\[
\norm{\nabla f(x)}\norm{x-x_p} \geq \langle \nabla f(x),x-x_p\rangle \geq \mu\norm{x_p-x}^2,
\]
and dividing both sides by $\norm{x-x_p}$ (assuming $x \not = x_p$) gives EB with the same $\mu$ (while EB  clearly holds if $x = x_p$).
\item $\bf EB \rightarrow PL$: By Lipschitz continuity we have
\[
f(x) \leq f(x_p) + \langle \nabla f(x_p),x-x_p\rangle + \frac L 2 \norm{x_p - x}^2,
\]
and using EB along with $f(x_p) = f^*$ and $\nabla f(x_p) =0$ we have
\[
f(x) - f^* \leq \frac L 2 \norm{x_p -x}^2 \leq \frac{L}{2\mu}\norm{\nabla f(x)}^2,
\]
which is the PL inequality with constant $\frac \mu L$.
\begin{comment}
\item $\bf RSI \rightarrow QG$: Following Zhang and Cheng (Lemma~\ref{lem:1}), using that any $y_\tau = x_p + \tau(x-x_p)$ projects onto $x_p$ we have from the integral form of Taylor's theorem that
\begin{align*}
f(x) & = f(x_p) + \int_0^1\langle \nabla f(x_p + \tau(x - x_p)),x-x_p)\rangle d\tau\\
& = f(x_p) + \int_0^1 \frac 1 \tau \langle \nabla f(x_p + \tau(x - x_p)),\tau(x-x_p)\rangle d\tau\\
& = f(x_p) + \int_0^1 \frac 1 \tau \langle \nabla f(x_p + \tau(x-x_p)),x_p + \tau(x-x_p) - x_p\rangle d\tau\\
& \geq f(x_p) + \int_0^1 \frac \mu \tau \norm{x_p + \tau(x-x_p) - x_p}^2 d\tau\\
& = f(x_p) + \int_0^1 \mu\tau\norm{x-x_p}^2 d\tau\\
& = f(x_p) + \frac \mu 2 \norm{x-x_p}^2,
\end{align*}
which is QG with the same $\mu$.
\end{comment}
\item $\bf PL \rightarrow EB$: Below we show that PL implies QG with the same constant. Using this result in PL we get
\[
\frac{1}{2}\| \nabla f(x) \|^2 \ge \mu (f(x) - f^*) \ge \frac{\mu^2}{2} \|x - x_p \|^2,
\]
which implies that EB holds with the same constant.
\item $\bf QG + Convex \rightarrow RSI$: By convexity we have
\[
f(x_p) \geq f(x) + \langle \nabla f(x),x_p -x\rangle.
\]
Re-arranging and using QG we get
\[
 \langle \nabla f(x),x - x_p \rangle \ge f(x) - f^* \ge \frac \mu 2 \norm{x_p - x}^2,
\]
which is RSI with constant $\frac{\mu}{2}$.
\item $\bf PL \rightarrow QG$: Our argument that this implication holds is similar to the argument used in related works~\citep{Bolte2015,Zhang2015}
Define the function
\[
	g(x) = \sqrt{f(x) - f^*}.
\]
If we assume that $f$ satisfies the PL inequality then for any $x \not\in \mathcal{X}^*$ we have
\[
	\norm{\nabla g(x) }^2 = \left|\left|\frac{1}{2\sqrt{f(x) - f^*}}\nabla f(x)\right|\right|^2= \frac{\norm{\nabla f(x)}^2}{4(f(x) - f^*)} \geq \frac{\mu}{2},
	\]
	or that
	\begin{equation}\label{eq:PLg}
	 \norm{\nabla g(x) } \geq \sqrt{\frac \mu 2}.
\end{equation}
By the definition of $g$, to show QG it is sufficient to show that
\begin{equation}\label{eq:result}
	g(x) \geq \sqrt{\frac \mu 2} \norm{x - x_p}.
\end{equation}
%Here we prove that if $g(x)$ is positive invex function, $g(x) \geq 0$, and it has a closed optimal solution set $S$ such that for all $y \in S$, $g(y) = 0$. If $\norm{\nabla g(x)} \geq \mu$ then $g(x) \geq \mu \norm{x - x_p}$.
As $f$ is assumed to satisfy the PL inequality we have that $f$ is an invex function and thus by definition $g$ is a positive invex function ($g(x) \geq 0$) with a closed optimal solution set $\mathcal{X}^*$ such that for all $y \in \mathcal{X}^*$, $g(y) = 0$. For any point $x_0 \not\in \mathcal{X}^*$, consider solving the following differential equation:
\begin{align}\label{eq:ode}
	\frac{d x(t)}{dt} & = -\nabla g(x(t)) \nonumber\\
	x(t = 0) & = x_0,
\end{align}
for $x(t) \not\in \mathcal{X}^*$.
(This is a flow orbit starting at $x_0$ and flowing along the gradient of $g$.) By \eqref{eq:PLg}, $\nabla g$ is bounded from below, and as $g$ is a positive invex function $g$ is also bounded from below. Thus, by moving along the path defined by \eqref{eq:ode} we are sufficiently reducing the function and will eventually reach the optimal set. Thus there exists a $T$ such that $x(T) \in \mathcal{X}^*$ (and at this point the differential equation ceases to be defined). We can show this by using the steps
\begin{align*}
\label{eq:eq1}
g(x_0) - g(x_t) & = \int_{x_t}^{x_0} \langle \nabla g(x), dx\rangle  &(\text{gradient theorem for line integrals}) \\
& = -\int_{x_0}^{x_t} \langle \nabla g(x),dx\rangle  &(\text{flipping integral bounds}) \\
& = - \int_0^T \langle \nabla g(x(t)), \frac{dx(t)}{dt} \rangle ~dt & (\text{reparameterization}) \\
(*) \hspace{50pt} & =  \int_0^T \norm{\nabla g(x(t))}^2 ~dt  & \text{(from \eqref{eq:ode})} \\
& \geq  \int_0^T \frac \mu 2 dt  & \text{(from \eqref{eq:PLg})} \\
& = \frac\mu 2 T.
\end{align*}
As $g(x_t) \geq 0$, this shows we need to have $T \le 2g(x_0)/\mu$, so there must be  a $T$ with $x(T) \in \mathcal{X}^*$. 

The \emph{length} of the orbit $x(t)$ starting at $x_0$, which we'll denote by $\mathcal{L}(x_0)$, is given by
\begin{equation}\label{eq:eq2}
	\mathcal{L}(x_0) = \int_0^T \norm{dx(t)/dt} dt = \int_0^T \norm{\nabla g(x(t))} ~dt \geq \norm{x_0 - x_p},
\end{equation}
where $x_p$ is the projection of $x_0$ onto $\mathcal{X}^*$ and the inequality follows because the orbit is a path from $x_0$ to a point in $\mathcal{X}^*$ (and thus it must be at least as long as the projection distance).
\begin{comment}
Let $x_p$ be the projection of $x_0$ onto the set $\mathcal{X}^*$. Since the orbit is a path from $x_0$ to a point in $\mathcal{X}^*$ its length must be greater than or equal to $\norm{x_0 - x_p}$. Thus we have
\begin{equation}
\label{eq:eq2}
\mathcal{L}(x_0) = \int_0^T \norm{\nabla g(x(t))} ~dt \geq \norm{x_0 - x_p}
\end{equation}
\end{comment}

Starting from the line marked $(*)$ above we have
\begin{align*}
g(x_0) - g(x_T) &= \int_0^T \norm{\nabla g(x(t))}^2 ~dt \\
& \geq \sqrt{\frac\mu 2} \int_0^T \norm{\nabla g(x(t))} ~dt & \text{(by the PL inequality variation in \eqref{eq:PLg})}\\
& \geq \sqrt{\frac\mu 2} \norm{x_0 - x_p}. & \text{(by \eqref{eq:eq2})} 
\end{align*} 
As $g(x_T) = 0$, this yields our result \eqref{eq:result}, or equivalently 
\[
	f(x) - f^* \ge \frac{\mu}{2} \| x - x_p \|^2,
\]
which is QG with the same constant.
\end{itemize}

\section{Relevant Problems}\label{app:problems}

\textbf{Strongly-convex}: \\ By minimizing both sides of the SC inequality with respect to $y$ we get
\[
f(x^*) \geq f(x) - \frac{1 }{ 2\mu}||\nabla f(x)||^2,
\]
which implies the PL inequality holds with the same value $\mu$. Thus,
%Note that since $\mu$ is the same, the $(1-\mu/L)$ rate from 
Theorem~\ref{th:lin_pl} exactly matches the known rate for gradient descent with a step-size of $1/L$ for a $\mu$-SC function.

\textbf{Strongly-convex composed with linear}: \\
To show that this class of functions satisfies the PL inequality, we first define $f(x) := g(Ax)$ for a $\sigma$-strongly convex function $g$. For arbitrary $x$ and $y$, we define $u := Ax$ and $v := Ay$. By the strong convexity of $g$, we have
\[
	g(v) \ge g(u) + \nabla g(u)^T(v - u) + \frac{\sigma}{2} \| v - u \|^2.
\]
By our definitions of $u$ and $v$, we get
\[
	g(Ay) \ge g(Ax) + \nabla g(Ax)^T(Ay - Ax) + \frac{\sigma}{2} \| Ay - Ax \|^2,
\]
where we can write the middle term as $(A^T \nabla g(Ax))^T(y - x)$. By the definition of $f$ and its gradient being $\nabla f(x) = A^T \nabla g(Ax)$ by the multivariate chain rule, we obtain
\[
f(y) \geq f(x) + \langle \nabla f(x) , y-x \rangle + \frac{\sigma}{2} || A(y-x)||^2.
\]
Using $x_p$ to denote the projection of $x$ onto the optimal solution set $\mathcal{X}^*$, we have
\begin{align*}
	f(x_{p}) 
	& \geq  f(x) + \langle \nabla f(x) , x_{p}-x \rangle + \frac{\sigma}{2} || A(x_{p}-x)||^2 \\
	& \geq f(x) + \langle \nabla f(x) , x_{p}-x \rangle + \frac{\sigma \theta(A)}{2} ||x_{p}-x||^2 \\
  	& \geq f(x)  + \min_y \left[\langle \nabla f(x) ,y-x \rangle + \frac{\sigma \theta(A)}{2} ||y-x||^2 \right] \\
  	& = f(x) - \frac{1}{2\theta(A)\sigma}||\nabla f(x)||^2.
\end{align*}
In the second line we use that $\mathcal{X}^*$ is polyhedral, and use the theorem of~\citet{hoffman} to obtain a bound in terms of $\theta(A)$ (the smallest non-zero singular value of $A$). This derivation implies that the PL inequality is satisfied with $\mu = \sigma\theta(A)$.

\section{Sign-Based Gradient Methods}\label{app:huge-scale}

The learning heuristic RPROP (Resilient backPROPagation) is a classic iterative method used for supervised learning problems in feedforward neural networks~\citep{redmiller92}. The general update for some vector of step sizes $\alpha_k \in \R^d$ is given by
\[
	x^{k+1} = x^k - \alpha^k \circ \sign{\nabla f(x^k)},
\]
where the $\circ$~operator indicates coordinate-wise multiplication.
Although this method has been used for many years in the machine learning community, we are not aware of any previous convergence rate analysis of such a method. Here we give a convergence rate when the individual step-sizes $\alpha_i^k$ are chosen proportional to $1/\sqrt{L_i}$, where the $L_i$ are constants such that the gradient is 1-Lipschitz continuous in the  norm defined by
\[
	\| z \|_{L^{-1}[1]} \triangleq \sum_i \frac{1}{\sqrt{L_i}} |z_i|.
\]
Formally, we assume that the $L_i$ are set so that for all $x$ and $y$ we have
\[
	\| \nabla f(y) - \nabla f(x) \|_{L^{-1}[1]} \le \| y - x \|_{L[\infty]},
\]
and where the dual norm of the $\|\cdot \|_{L^{-1}[1]}$ norm above is given by the $\|\cdot\|_{L[\infty]}$ norm,
\[
	\| z \|_{L[\infty]} \triangleq \max_i \sqrt{L_i} |z_i |.
\]
%where $L_i$ is the coordinate-wise Lipschitz constant of $\nabla_i f$,
%\[
%	| \nabla_i f(x + h e_i) - \nabla_i f(x) | \le L_i \cdot |h|, \quad \forall~x \in \R^d, ~h \in \R.
%\]
%Note that this is different than the coordinate-wise Lipschitz assumption we used to analyze coordinate descent methods in Section~\ref{sec:coo-des}, since it bounds the maximum change in the gradient of coordinate $i$ based on changes in any coordinate (not just coordinate $i$). 
We note that such $L_i$ always exist if the gradient is Lipschitz continuous, so this is not adding any assumptions on the function $f$. % and we assume that $\nabla f$ is Lipschitz contiuous with respect to the $L[\infty]$-norm, which means that for all $x, y \in \R^d$,
The particular choice of the step-sizes $\alpha_i^k$ that we will analyze
 is
 \[
	\alpha^k_i = \frac{\| \nabla f(x^k) \|_{L^{-1}[1]}}{\sqrt{L_i}}, 
\]
which yields a linear convergence rate for problems where the PL inequality is satisfied. %Note that we can show convergence if $\alpha_i^k$ is the same for all $i$ if we take $L_i = \max_i\{L_i\}$. 

The coordinate-wise iteration update under this choice of $\alpha_i^k$ is given by
\[
	x^{k+1}_i = x^k_i - \frac{\| \nabla f(x^k) \|_{L^{-1}[1]}}{\sqrt{L_i}}  \sign{\nabla_i f(x^k)}.
\]
%where the $\circ$-operator is used for represent component-wise multiplication.
Defining a diagonal matrix $\Lambda$ with $1/\sqrt{L_i}$ along the diagonal, the update can be written as
\[
	x^{k+1} = x^k - \| \nabla f(x^k) \|_{L^{-1}[1]} \Lambda \circ \sign{\nabla f(x^k)}.
\]
%A special case of this update occurs when all the $L_i$ are equal yielding
%\[
%	x^{k+1} = x^k - \frac{\| \nabla f(x^k) \|_1}{L_\infty} \sign{\nabla f(x^k)},
%\]
%where, in this case, $\nabla f$ is Lipschitz continuous with respect to the $L[\infty]$-norm,
%\[
%	\| \nabla f(y) - \nabla f(x) \|_{L^{-1}[1]} \le L_\infty \| y - x \|_{L[\infty]}.
%\]
%We note that when $\nabla f(x^k) \not = 0$, $- \sign{\nabla f(x^k)}$ is a descent direction as
%\[
%	\langle - \sign{\nabla f(x^k)}, \nabla f(x^k) \rangle = - \| \nabla f(x^k) \|_1 < 0.
%\]
Consider the function $g(\tau) = f(x + \tau(y - x))$ with $\tau \in \R$. Then
\begin{align*}
	f(y) - f(x) - \langle \nabla f(x), y - x \rangle 
	&= g(1) - g(0) - \langle \nabla f(x), y - x \rangle\\
	&= \int_0^1 \frac{dg}{d\tau}(\tau) - \langle \nabla f(x), y - x \rangle ~d\tau \\
	&= \int_0^1 \langle \nabla f(x + \tau(y - x)), y - x \rangle - \langle \nabla f(x), y - x \rangle ~d\tau \\
	&= \int_0^1 \langle \nabla f(x + \tau(y - x)) - \nabla f(x), y - x \rangle ~d\tau \\
	&\le \int_0^1 \| \nabla f(x + \tau(y - x)) - \nabla f(x) \|_{L^{-1}[1]} \|y - x \|_{L[\infty]} ~d\tau \\
	&\le \int_0^1 \tau \| y - x \|^2_{L[\infty]} ~d\tau \\
	&= \tau^2 \frac{1}{2} \| y - x \|^2_{L[\infty]} \bigg |^1_0 \\
	&= \frac{1}{2} \| y - x \|^2_{L[\infty]}\\
	&= \frac{1}{2} | y - x \|^2_{L[\infty]}.
%	&= \int_0^1 \left(\sum_i \frac{1}{\sqrt{L_i}} | \nabla_i f(x + \tau(y - x)) - \nabla_i f(x) |\right) \|y - x \|_{L[\infty]} ~d\tau \\
%	&\le \int_0^1 \left( \sum_i \frac{L_i}{\sqrt{L_i}} | x + \tau(y - x) - x | \right)\|y - x \|_{L[\infty]} ~d\tau \\
%	&= \int_0^1 \tau \left(\sum_i \sqrt{L_i} | y - x|\right) \|y - x \|_{L[\infty]} ~d\tau \\
%	&\le \int_0^1 \tau d \| y - x \|^2_{L[\infty]} ~d\tau \\
%	&= \tau^2 \frac{d}{2} \| y - x \|^2_{L[\infty]} \bigg |^1_0 \\
%	&= \frac{d}{2} \| y - x \|^2_{L[\infty]},
%%	&= \frac{1}{2} \| y - x \|^2_{L[\infty]}.
\end{align*}
where the second inequality uses the Lipschitz assumption, and in the first inequality we've used the Cauchy-Schwarz inequality and that the dual norm of the $L^{-1}[1]$ norm is the $L[\infty]$ norm. The above gives an upper bound on the function in terms of this $L[\infty]$-norm,
\[
	f(y) \le f(x) + \langle \nabla f(x), y - x \rangle + \frac{1}{2} \| y - x \|^2_{L[\infty]}.
\]
Plugging in our iteration update we have 
\begin{align*}
	&f(x^{k+1}) \\
	&~\le f(x^k) + \langle \nabla f(x^k), x^{k+1} - x^k \rangle + \frac{1}{2} \| x^{k+1} - x^k \|^2_{L[\infty]} \\
	&~= f(x^k) - \| \nabla f(x^k) \|_{L^{-1}[1]} \langle \nabla f(x^k), \Lambda \circ \sign{\nabla f(x^k)} \rangle + \frac{\| \nabla f(x^k) \|^2_{L^{-1}[1]}}{2} \| \Lambda \circ \sign{\nabla f(x^k)} \|^2_{L[\infty]} \\
	&~= f(x^k) - \| \nabla f(x^k) \|^2_{L^{-1}[1]} + \frac{\| \nabla f(x^k) \|^2_{L^{-1}[1]}}{2} \bigg ( \max_i \frac{1}{\sqrt{L_i}} \sqrt{L_i}|\sign{\nabla_i f(x^k)}| \bigg )^2 \\
	&~= f(x^k) - \frac{1}{2}\| \nabla f(x^k) \|^2_{L^{-1}[1]}.
\end{align*}
Subtracting $f^*$ from both sides yields
\[
	f(x^{k+1}) - f(x^*) \le f(x^k) - f(x^*) - \frac{1}{2}\| \nabla f(x^k) \|^2_{L^{-1}[1]}.
\]
Applying the PL inequality with respect to the ${L^{-1}[1]}$-norm (which, if the PL inequality is satisfied, holds for some $\mu_{L[\infty]}$ by the equivalence between norms), 
\[
	\frac{1}{2} \| \nabla f(x^k) \|^2_{L^{-1}[1]} \ge \mu_{L[\infty]} \left ( f(x^k) - f^* \right ),
\]
we have
\[
	f(x^{k+1}) - f(x^*) \le \left ( 1 - \mu_{L[\infty]} \right ) \left (f(x^k) - f(x^*)  \right ).
\]

\section{Linear Convergence Rate of SVRG Method}\label{app:svrg}
In this section, we look at the SVRG method for the finite-sum optimization problem,
\be
f(w) = \frac{1}{n} \sum_i f_i(w). 
\ee
To minimize functions of this form, the SVRG algorithm of~\cite{johnson2013stochastic} uses iterations of the form
\be
x_t = x_{t-1} -\alpha [\nabla f_{i_t}(x_{t-1})  -   f_{i_t}(x^s) + \mu^s],
\ee
where $i_t$ is chosen uniformly from $\{1,2,\dots,n\}$ and we assume the step-size satisfies $\alpha < 2/L$.
In this algorithm we start with some $x^0$ and initially set $\mu^0 = \nabla f(x^0)$ and $x_0 = x^0$, but after every $m$ steps we set $x^{s+1}$ to a random $x_t$ for $t \in \{ms + 1,\dots,m(s+1)\}$, then replace $\mu^s$ with $\nabla f(x^s)$ and $x^t$ with $x^{s+1}$. Analogous to~\citet{johnson2013stochastic} for the SC case, we now show tnat SVRG has a linear convergence rate if each $f_i$ is a convex function with a Lipschitz-continuous gradient and $f$ satisfies the PL inequality.

Following the same argument as~\citet{johnson2013stochastic}, for any solution $x^*$  the assumptions on the $f_i$ mean that the ``outer" SVRG iterations $x^s$ satisfy
\[
2\alpha(1-2L\alpha)m\mathbb{E}[f(x^s) - f^*] \leq \mathbb{E}[\norm{x^{s-1} - x^*}^2] + 4L\alpha^2m\mathbb{E}[f(x^{s-1}) - f^*].
\]
Choosing the particular $x^*$ that is the projection of $x^{s-1}$ onto the solution set and using QG (which is equivalent to PL in this convex setting) we have
\[
2\alpha(1-2L\alpha)m\mathbb{E}[f(x^s) - f^*] \leq \frac{2}{\mu}\mathbb{E}[f(x^{s-1}) - f^*] + 4L\alpha^2m\mathbb{E}[f(x^{s-1}) - f^*].
\]
Dividing both sides by $2\alpha(1-2L\alpha)m$ we get
\[
\mathbb{E}[f(x^s) - f^*]  \leq \frac{1}{1-2\alpha L}\left(\frac{1}{m\mu\alpha} + 2L\alpha\right)\mathbb{E}[f(x^{s-1}) - f^*],
\]
which is a linear convergence rate for sufficiently large $m$ and sufficiently small $\alpha$.

\section{Proximal-PL Lemma}\label{app:prox-pl}

In this section we give a useful property of the function $\mathcal{D}_g$.
\begin{lemma}\label{lem:1}
For any differentiable function $f$ and any convex function $g$, given $\mu_2 \geq \mu_1 > 0$ we have
\[
 \mathcal{D}_g(x,\mu_2) \geq  \mathcal{D}_g(x,\mu_1).
\]
\end{lemma}
%We first prove this lemma and then use it to show that the proximal PL-inequality is satisfied for the functions listed at the end of Section~\ref{sec:prox-pl}, followed by the proof for randomized coordinate descent.
We'll prove Lemma~\ref{lem:1} as a corollary of a related result. We first restate the definition
\be 
\label{prox0}
\mathcal{D}_g(x,\lambda) = -2\lambda \min_y \left[ \langle \nabla f(x) , y-x \rangle + \frac{\lambda}{2}||y-x||^2 + g(y) - g(x) \right],
\ee
and we note that we require $\lambda > 0$. By completing the square, we have
\begin{align*}
\label{prox1}
\mathcal{D}_g(x,\lambda) & = - \min_y \left[-\norm{\nabla f(x)}^2 + \norm{\nabla f(x)}^2 + 2\lambda\langle \nabla f(x) , y-x \rangle + \lambda^2||y-x||^2 + 2\lambda(g(y) - g(x) )\right]\\
& = ||\nabla f(x) || ^2 - \min_y \left[ || \lambda(y -x ) + \nabla f(x) ||^2 + 2\lambda (g(y) - g(x) )\right].
\end{align*}
Notice that if $g = 0$, then $\mathcal{D}_g(x,\lambda) = ||\nabla f(x) || ^2 $ and the proximal-PL inequality reduces to the PL inequality. We'll the define the \emph{proximal residual} function as the second part of the above equality,
\be
\label{Res}
\mathcal{R}_g(\lambda, x,a) \triangleq  \min_y \left[ || \lambda(y -x ) + a ||^2 + 2\lambda (g(y) - g(x) \right].
\ee

\begin{lemma}\label{lem:2} If $g$ is convex then for any $x$ and $a$, and for $0 < \lambda_1 \leq \lambda_2$ we have
\be
\mathcal{R}_g(\lambda_1, x,a) \geq  \mathcal{R}_g(\lambda_2, x,a).
\ee
\end{lemma}
\begin{proof}
%If $g$ is convex then for any $x$ and $a$, given $\lambda_1 \leq \lambda_2$ we have
%\be
%\mathcal{R}_g(\lambda_1, x,a) \geq  \mathcal{R}_g(\lambda_2, x,a).
%\ee
%From equation \eqref{Res}, we have
%\be
%\mathcal{R}_g(\lambda, x,a) =  \min_y \left[ || \lambda(y -x ) + a ||^2 + 2\lambda (g(y) - g(x) \right]
%\ee
Without loss of generality, assume $x = 0$. Then we have
\bea
 \mathcal{R}_g(\lambda ,a) & = &  \min_y \left[ || \lambda y  + a ||^2 + 2\lambda (g(y) - g(0) \right] \nonumber \\
  \label{eq:29}
  & = & \min_{\bar{y}} \left[ || \bar{y}  + a ||^2 + 2\lambda (g(\bar{y}/\lambda) - g(0) \right],
\eea
where in the second line we used a changed of variables $\bar{y} = \lambda y$ (note that we are minimizing over the whole space of $\R^n$). By the convexity of $g$, for any $\alpha \in [0,1]$ and $z\in\R^n$ we have
\begin{align}
\label{cx2}
g( \alpha z ) &\leq \alpha g(z) + (1-\alpha) g(0) \nonumber \\
\iff \quad g(\alpha z)  - g(0) &\leq \alpha (g(z) - g(0)).
\end{align}
By using $0 < \lambda_1/\lambda_2 \leq 1$ and using the choices $\alpha =\frac{\lambda_1 }{ \lambda_2}$ and $z = {\bar{y} / \lambda_1}$ we have
\begin{align}
 g(\bar{y}/\lambda_2) - g(0) &\leq \frac{\lambda_1 }{ \lambda_2}(g(\bar{y}/\lambda_1) - g(0)) \nonumber \\
\iff \quad  \lambda_2(g(\bar{y}/\lambda_2) - g(0)) &\leq {\lambda_1}(g(\bar{y}/\lambda_1) - g(0)),
\end{align}
Adding $||\bar{y} + a||^2$ to both sides, we get
\be
 ||\bar{y} + a||^2 +  \lambda_2(g(\bar{y}/\lambda_2) - g(0)) \leq||\bar{y} + a||^2 + {\lambda_1}(g(\bar{y}/\lambda_1) - g(0)).
\ee
Taking the minimum over both sides with respect to $\bar{y}$ yields Lemma~\ref{lem:2} due to~\eqref{eq:29}. \qed
\end{proof}

\begin{corollary}\label{cor:1}
 For any differentiable function $f$ and convex function $g$, given $\lambda_1 \leq \lambda_2$, we have
\be
\mathcal{D}_g(x,\lambda_2) \geq \mathcal{D}_g(x,\lambda_1).
\ee
\end{corollary}
By using $\mathcal{D}_g(x,\lambda) = ||\nabla f(x) || ^2 - \mathcal{R}_g(\lambda, x,\nabla f(x))$, Corollary~\ref{cor:1} is exactly Lemma~\ref{lem:1}.
\\ \\

\section{Relevant Problems}\label{app:probs}
In this section we prove that the three classes of functions listed in Section~\ref{subsec:probs} satisfy the proximal-PL inequality condition. Note that while we prove these hold for $\mathcal{D}_g(x,\lambda)$ for $\lambda \leq L$, by Lemma~\ref{lem:1} above they also hold for $\mathcal{D}_g(x,L)$.

\begin{enumerate}
\item {\em $f(x)$, where $f$ satisfies the PL inequality ($g$ is constant)}: \\
	As $g$ is assumed to be constant, we have $g(y) - g(x) = 0$ and the left-hand side of the proximal-PL inequality simplifies to
	\begin{align*}
		\mathcal{D}_g(x,\mu) 
		& = -2 \mu \min_y \left \{ \langle \nabla f(x), y- x \rangle + \frac{\mu}{2} \| y - x \|^2 \right \} \\
		& = -2\mu \left(- \frac{1}{2\mu}\norm{f(x)}^2\right)\\
%		&= 2 \mu \max_y \left \{ \langle -\nabla f(x), y- x \rangle - \frac{\mu}{2} \| y - x \|^2 \right \} \\
		&= \| \nabla f(x) \|^2,
	\end{align*}
	%where the final line follows from the definition of the convex conjugate of the squared $\ell_2$-norm. 
	Thus, the proximal PL inequality simplifies to $f$ satisfying the PL inequality,
	\[
		\frac{1}{2} \| \nabla f(x) \|^2 \ge \mu \left (f(x) - f^* \right),
	\]
	as we assumed.
%	As $f$ is assumed to satisfy the PL inequality for some $\mu >0$, then we have our result.
	
\item {\em $F(x) = f(x)  + g(x)$ and $f$ is strongly convex}:\\
By the strong convexity of $f$ we have
\be
f(y) \geq f(x) + \langle \nabla f(x), y -x\rangle + \frac{\mu }{ 2} || y-x ||^2,
\ee
which leads to
\be 
F(y) \geq F(x) + \langle \nabla f(x), y -x\rangle + \frac{\mu }{ 2} || y-x ||^2  + g(y) - g(x).
\ee
Minimizing both sides respect to $y$, 
\bea 
F^* & \geq & F(x) + \min_{y} \langle \nabla f(x), y -x\rangle + \frac{\mu }{ 2} || y-x ||^2  + g(y) - g(x) \nonumber \\
& = & F(x) - \frac{1 }{ 2\mu}\mathcal{D}_g(x,\mu).
\eea
Rearranging, we have our result.

\item {\em $F(x) = f(Ax)  + g(x)$ and $f$ is strongly convex, $g$ is the indicator function for a polyhedral set $\mathcal{X}$, and $A$ is a linear transformation}: \\
By defining $\tilde{f}(x)  = f(Ax) $ and using strong convexity of $f$, we have
\be
\tilde{f}(y) \geq \tilde{f}(x) +  \langle \nabla \tilde{f}(x), y -x\rangle + \frac{\mu }{ 2} || A( y-x) ||^2,
\ee
which leads to
\be 
F(y) \geq F(x) + \langle \nabla \tilde{f}(x), y -x\rangle + \frac{\mu }{ 2} || A(y-x) ||^2  + g(y) - g(x).
\ee
Since $\mathcal{X}$ is polyhedral, it can be written as a set  $\{x : Bx \le c \}$ for a matrix $B$ and a vector $c$. As before, assume that $x_p$ is the projection of $x$ onto the optimal solution set $\mathcal{X}^*$ which in this case is $\{x : Bx \le c, Ax = z\}$ for some $z$.
\bea
F^*  & = & F(x_p) \geq F(x) + \langle \nabla \tilde{f}(x), x_p -x\rangle + \frac{\mu }{ 2} || A(x-x_p) ||^2  + g(x_p) - g(x) \nonumber \\
& = & F(x) + \langle \nabla \tilde{f}(x), x_p -x\rangle + \frac{\mu }{ 2} || Ax - z ||^2  + g(x_p) - g(x) \nonumber \\
& =  & F(x) + \langle \nabla \tilde{f}(x), x_p -x\rangle + \frac{\mu }{ 2} || \{ Ax - z \}_+ + \{-Ax + z\}_+ ||^2  + g(x_p) - g(x) \nonumber \\
& =  & F(x) + \langle \nabla \tilde{f}(x), x_p -x\rangle + \frac{\mu }{ 2} \left  \| \left \{  \left [ \begin{matrix} A\\ -A\\ B \end{matrix} \right ] x - \left [ \begin{matrix} z \\ -z \\ c \end{matrix} \right ] \right \}_+ \right \|^2  + g(x_p) - g(x) \nonumber \\
& \geq &  F(x) + \langle \nabla \tilde{f}(x), x_p -x\rangle + \frac{\mu \theta (A,B) }{ 2} || x - x_p ||^2  + g(x_p) - g(x) \nonumber \\
& \geq & F(x) + \min_y \left[ \langle \nabla \tilde{f}(x), y -x\rangle + \frac{\mu \theta (A,B) }{ 2} || y-x ||^2  + g(y) - g(x) \right] \nonumber \\
& = & F(x) - \frac{1 }{ 2 \mu \ \theta(A)} \mathcal{D}_g(x,\mu \theta(A,B)).
\eea
where we've used the notation that $\{\cdot\}_+ = \max\{0,\cdot\}$, the fourth equality follows because $x$ was projected onto~$\mathcal{X}$ in the previous iteration (so $Bx -c \leq 0$), and the line after that uses Hoffman's bound~\citep{hoffman}.%[see][Lemma~2.1]{luo1992error}.%, and the last line yields our result. %$\mathcal{D}_g(x,\mu \theta(A)) \geq 2 \mu \theta(A) (F(x) - F^*) $. \\

\item {\em $F(x) = f(x)  + g(x)$, $f$ is convex, and $F$ satisfies the quadratic growth (QG) condition}: \\
A function $F$ satisfies the QG condition if
\be
F(x) - F^* \geq \frac{\mu }{ 2} || x - x_p||^2.
\ee
For any $\lambda > 0$ we have,
\bea
&&\min_y \left[ \langle \nabla {f}(x), y -x\rangle + \frac{\lambda }{ 2} || y-x ||^2  + g(y) - g(x) \right] \nonumber \\
& \leq & \langle \nabla {f}(x), x_p -x \rangle + \frac{\lambda }{ 2} || x_p-x ||^2  + g(x_p) - g(x) \nonumber \\
& \leq & f(x_p) - f(x) + \frac{\lambda }{ 2} || x_p-x ||^2  + g(x_p) - g(x) \nonumber \\
& = & \frac{\lambda }{ 2} || x_p-x ||^2 + F^* - F(x) \nonumber \\
& \leq & \left (1 - \frac{\lambda }{ \mu} \right) (F^* -F).
\eea 
The third line follows from the convexity of $f$, and the last inequality uses the QG condition of $F$. Multiplying both sides by $-2\lambda$, we have
\be
\mathcal{D}_g(x,\lambda) = -2\lambda \min_y \left[ \langle \nabla \tilde{f}(x), y -x\rangle + \frac{\lambda }{ 2} || y-x ||^2  + g(y) - g(x) \right] \geq 2\lambda \left (1 - \frac{\lambda }{ \mu} \right ) (F(x) - F^*).
\ee
This is true for any $\lambda >0$, and by choosing $\lambda = \mu/2$ we have
\be
\mathcal{D}_g(x,\mu/2) \geq \frac{\mu }{ 2} (F(x) - F^*).
\ee

\item {\em $F$ satisfies the KL inequality or the proximal-EB inequality}:\\ In the next section we show that these are equivalent to the proximal-PL inequality.
\end{enumerate}

\section{Equivalence of Proximal-PL with KL and EB}
\label{app:proximalPLequiv}

The equivalence of the KL condition and the proximal-gradient variant of the Luo-Tseng EB condition is known for convex $f$, see~\cite[Corollary~3.6]{drusvyatskiy2016error} and the proof of \cite[Theorem~5]{Bolte2015}. Here we prove the equivalence of these conditions with the proximal-PL inequality for non-convex $f$. First we review the definitions of the three conditions:
\begin{enumerate}
\item \textbf{Proximal-PL}: There exists a $\mu > 0$ such that
\[
    \frac12 {\cal D}_{g}(x,L) \ge \mu (F(x) - F_*)
\]
where
\[
    {\cal D}_{g}(x,L) = -2L \min_y \left\{\langle\nabla f(x),y-x\rangle + \frac{L}2\|y - x\|^2 + g(y) - g(x)\right\}.
\]
%and for this equivalence we'll restrict to a sublevel set $F(x)\le \zeta$.
\item \textbf{Proximal-EB}: There exists $c > 0$ such that we have
  \begin{equation}\label{gLT}
  \norm{x - x_p} \le c \left\|x - {\rm prox}_{\frac1L g}\left(x - \frac1L \nabla f(x)\right)\right\|.
  \end{equation}
  %whenever $F(x)\le \zeta$.
\item \textbf{Kurdyka-{\L}ojasiewicz}: The KL condition with exponent ${1 \over 2}$ holds if there exist $\tilde{\mu} > 0$ such that
\be\label{KL}
\min_{s \in \partial F(x)} \norm{s}^2 \geq 2 \tilde{\mu} (F(x) - F_*)
\ee  
where $\partial F(x)$ is the Frechet subdifferential. In particular, if $F : H \rightarrow \mathcal{R}$ is a real-valued function then we say that $s \in H$ is a Frechet subdifferential of $F$ at  $x \in \hbox{dom} \ F$ if 
\be
\liminf_{y\rightarrow x,y\neq x} {F(y) - F(x) - \langle s, y-x \rangle \over \|y-x\|^2} \geq 0.
\ee
Note that for differentiable $f$ the Frechet subdifferential only contains the gradient, $\nabla f(x)$. In our case where $F(x) = f(x) + g(x)$ with a differentiable $f$ and a convex $g$ we have
\[ \partial F(x) = \{\nabla f(x) + \xi \; | \; \xi \in \partial g(x)\}.\]
The KL inequality is an intuitive generalization of the PL inequality since, analogous to the gradient vector in the smooth case, the negation of the quantity $\argmin{s\in \partial F(x)}\norm{s}$ points in the direction of steepest descent~\citep[see][Section 8.4]{bertsekas2003convex}
\end{enumerate}

\noindent We first derive an alternative representation of ${\cal D}_g(x,L)$ in terms of the so-called forward-backward envelope $F_{\frac1L}$ of $F$~\citep[see][Definition~2.1]{stella2016forward}.
Indeed,
\begin{equation}\label{DgxL}
  \begin{split}
    {\cal D}_{g}(x,L) &= -2L \min_y \left\{\langle\nabla f(x),y-x\rangle + \frac{L}2\|y - x\|^2 + g(y) - g(x)\right\}\\
    &= -2L \left[\min_y \left\{f(x) + \langle\nabla f(x),y-x\rangle + \frac{L}2\|y - x\|^2 + g(y)\right\} - f(x) - g(x)\right]\\
    &= -2L [F_{\frac1L}(x) - F(x)] = 2L [F(x) - F_{\frac1L}(x)],
  \end{split}
\end{equation}
It follows from the  definition of $F_{\frac1L}(x) $ that we have
\bea
\label{bound_on_F_L}
 F_{\frac1L}(x) - F^* & =& \min_y \left\{f(x) + \langle\nabla f(x),y-x\rangle + \frac{L}2\|y - x\|^2 + g(y)\right\} - f(x^*) - g(x^*) \nonumber \\
 & \leq & f(x) + \langle\nabla f(x),x^*-x\rangle + \frac{L}2\| x^* - x\|^2 + g(x^*) - f(x^*) - g(x^*) \nonumber \\
 & =  &f(x) - f(x^*) + \langle\nabla f(x),x^*-x\rangle + \frac{L}2\| x^* - x\|^2 \nonumber \\
 & =  &f(x) - f(x^*) + \langle\nabla f(x),x^*-x\rangle + \frac{L}2\| x^* - x\|^2 \nonumber \\
 & \leq & 2L\| x^* - x\|^2,
\eea
where the second line uses that we are taking the minimizer and the last line uses the Lipschitz continuity of $\nabla f$ as follows,
\begin{equation}
\label{eq:LipschitzWeird}
\begin{aligned}
f(x) - f(y) + \langle \nabla f(x), y -x \rangle & \leq \langle \nabla f(y),x-y\rangle + \frac{L}{2}\norm{y-x}^2 + \langle \nabla f(x),y-x\rangle\\
& = \langle \nabla f(y) - \nabla f(x), y-x\rangle + \frac{L}{2} \norm{y-x}^2\\
& \leq \norm{\nabla f(y) - \nabla f(x)}\norm{y-x} + \frac{L}{2} \norm{y-x}^2 \leq \frac{3L}{2}\norm{y-x}^2.
\end{aligned}
\end{equation}

\begin{itemize}
\item \textbf{Proximal-EB $\rightarrow$ proximal-PL}: we have that
\bea
F(x) - F^*  &=& F(x) - F_{\frac1L}(x) + F_{\frac1L}(x) - F^* \nonumber \\
&\leq & F(x) - F_{\frac1L}(x) + 2L \| x^* - x\|^2 \nonumber \\
&\leq & F(x) - F_{\frac1L}(x) + C_0 \left\|x - {\rm prox}_{\frac1L g}\left(x - \frac1L \nabla f(x)\right)\right\|^2 \nonumber \\
&\leq &  C_1 (F(x) - F_{\frac1L}(x)),
\eea
for some constants $C_0$ and $C_1$, 
where the second inequality uses the proximal-EB and the last inequality follows from~\citet[Proposition~2.2(i)]{stella2016forward}. Now by using the fact that $F(x) - F_{\frac1L}(x) = {1 \over 2L}{\cal D}_{g}(x,L) $, the function satisfies the proximal-PL inequality.

\item \textbf{Proximal-PL $\rightarrow$ KL}: It's sufficient to prove that ${\cal D}_{g}(x,\mu) \leq \min_{s \in \partial F(x)} \norm{s}^2$ for any $x$ and $\mu$. First we observe that for any subgradient $\xi \in \partial g(x)$ we have
\bea
\langle\nabla f(x),y-x\rangle + \frac{\mu}2\|y - x\|^2 + g(y) - g(x) &\geq& \\ \nonumber
  \langle\nabla f(x),y-x\rangle + \frac{\mu}2\|y - x\|^2 + \langle \xi , y-x\rangle  &=& \\ \nonumber
  \langle\nabla f(x) + \xi,y-x\rangle + \frac{\mu}2\|y - x\|^2,
\eea
where the inequality follows from the definition of a subgradient.  Now by minimizing both sides over $y$ we have
\bea
\min_y \left\{\langle\nabla f(x),y-x\rangle + \frac{\mu}2\|y - x\|^2 + g(y) - g(x) \right\} \geq \\ \nonumber \min_y \left\{ \langle\nabla f(x) + \xi,y-x\rangle + \frac{\mu}2\|y - x\|^2 \right\} \geq 
-{1 \over 2\mu} \|\nabla f(x) + \xi \|^2.
\eea
Multiplying both sides with $-2\mu$ we get
\be
  {\cal D}_{g}(x,\mu) \leq \|\nabla f(x) + \xi \|^2
\ee
Since this holds for any $\xi \in \partial g(x)$, then it holds for any $\zeta = \nabla f(x) + \xi \in \partial F(x)$, it also holds for the minimium-norm subgradient of $F$.
\item \textbf{KL $\rightarrow$ Proximal-EB}: A function $h(x)$ is called ``semiconvex'' if there exist an $\alpha > 0$ such that $h(x) + \alpha \| x\|^2$ is convex~\citep[see][Definition 10]{bolte2010characterizations}. Note that Lipschitz-continuity of $\nabla f$ implies semi-convexity of $f$ in light of~\eqref{eq:LipschitzWeird} and~\citet[][Remark 11(iii)]{bolte2010characterizations}. It follows from convexity of $g$ that $F$ is semi-convex.
From~\citet[][Theorem 13]{bolte2010characterizations}, for any $x \in \hbox{dom} F$ there exist a subgradient curve $\chi_x : [0, \infty] \rightarrow \hbox{dom} F $ that satisfies
 \bea
 \label{subG}
 \dot{\chi_x}(t) & \in & -\partial F(\chi_x(t)) \nonumber \\
 \chi_x(0) &=& x \nonumber \\
 { d \over dt} F(\chi_x(t)) &=& - \|  \dot{\chi_x}(t) \|^2,
 \eea
 where $F(\chi_x(t))$ is non-increasing and Lipschitz continuous on $[\eta, \infty]$ for any $\eta >0$. By using these facts let's define the function $r(t) = \sqrt{F(\chi_x(t)) - F^*}$. It is easy to see that
\bea 
 {dr(t) \over dt} &=& {\dot{F}(\chi_x(t)) \over 2 \sqrt{F(\chi_x(t)) - F^*} } \nonumber \\
 &=& -{ \|  \dot{\chi_x}(t) \|^2 \over 2 \sqrt{F(\chi_x(t)) - F^*} } \nonumber \\
 &\leq& - \sqrt{\tilde{\mu}/2} \|  \dot{\chi_x}(t) \|,
 \eea
 for the second line we used the definition of subgradient curve and for the third line we used KL inequality condition and the fact that $\dot{\chi_x}(t)  \in  -\partial F(\chi_x(t))$. \\
 Now we have
 \bea
 \label{integral_1}
 r(T) - r(0) &=& \int_0^T {d \over dt}r(t) dt \nonumber \\
 &\leq & - \sqrt{\tilde{\mu}/2}\int_0^T  \|  \dot{\chi_x}(t) \| dt \nonumber \\
 &=& -\sqrt{\tilde{\mu}/2}{\rm dist} (\chi_x(T),\chi_x(0)),
 \eea
where we used the bound on the derivative of $r(t)$ above and that the length of the curve connecting any two points is less than the Euclidean distance between them. We're now going to take the limit of $T \rightarrow \infty$, while using the facts that $r(\infty) = 0$ (which we prove below) and using $r(0) = \sqrt{F(x) - F^*}$. This gives
\be
\sqrt{F(x) - F^*} \geq  \sqrt{\tilde{\mu}/2} \hbox{dist}(x, \mathcal{X}) 
\ee
From this inequality and also KL condition \ref{KL}, we proved that there exist a $C >0$ such that 
\be 
\label{derivative_bound}
{\rm dist}(0,\partial F(x)) \geq C {\rm dist}(x, \mathcal{X}).
\ee

Now let's show that $r(\infty) = 0$ or $\chi_x(\infty) \in \mathcal{X}$. From equation \ref{integral_1} we have
\bea
\label{int_ineq}
  r(T) - r(0) &=& \int_0^T {d \over dt}r(t) dt \nonumber \\
  &=& - \int_0^T { \|  \dot{\chi_x}(t) \|^2 \over 2 \sqrt{F(\chi_x(t)) - F^*} } \nonumber \\
  &\leq& -{\tilde{\mu} \over 2} \int_0^T  \sqrt{F(\chi_x(t)) - F^*} \nonumber \\
  &\leq& -{\tilde{\mu} \over 2} \int_0^T  \sqrt{F(\chi_x(T)) - F^*} \nonumber \\
  &=& -{\tilde{\mu}\sqrt{F(\chi_x(T)) - F^*} \over 2} T = -{\tilde{\mu} \: T \: r(T) \over 2}
\eea
where for the first inequality we used the KL property, and for the second inequality we used the fact that $F(\chi_x(t))$ is non-increasing, which means $F(\chi_x(T)) \leq F(\chi_x(t))$ for any $t \in [0,T]$. This inequality gives a bound on $r(T)$,
\[0 \leq r(T) \leq {2r(0) \over 2+\tilde{\mu}T }\]
now by taking the limit of $T \rightarrow \infty$ , we get $r(T) \rightarrow 0$. \\

Now by using \ref{derivative_bound} we can show that the proximal-EB condition is satisfied. Let's define $\hat{x} = {\rm prox}_{\frac1L g}\left(x - \frac1L \nabla f(x)\right) $. From the optimality of $\hat{x}$ we have $-\nabla f(x) - L(\hat{x} - x) \in \partial g(\hat{x})$, using this we get 
\[\nabla f(\hat{x})-\nabla f(x) - L(\hat{x} - x) \in \partial g(\hat{x}) + \nabla f(\hat{x}) = \partial F(\hat{x}).\]
Denoting the particular subgradient of $g$ that achieves this by $\xi$, we have
\bea
\label{partial_bound}
{\rm dist}(0, \partial F(\hat{x}))  &\leq& \|0 - \xi\| \nonumber\\
&=&\| \nabla f(\hat{x})-\nabla f(x) - L(\hat{x} - x)\| \nonumber \\
&\leq& L \|(\hat{x} - x)\| +  \| \nabla f(\hat{x})-\nabla f(x)\|\nonumber \\
&\leq&  2L \|(\hat{x} - x)\|,
\eea
where the second inequality uses $\| a -b\| \leq \|a\| + \|b\|$, and for the last line we used Lipschitz continuity of of $\nabla f$. We finally get the proximal-EB condition using
\bea 
{\rm dist}(x,\mathcal{X}) &\leq & \| x - \hat{x}\| + {\rm dist}(\hat{x},\mathcal{X}) \nonumber \\
&\leq & \| x - \hat{x}\|  + C {\rm dist}(0, \partial F(\hat{x}) \nonumber \\
&\leq & \| x - \hat{x}\|  + 2CL \|(\hat{x} - x)\| \nonumber \\
&=& (1 + 2CL) \| x - \hat{x}\|, 
\eea
where the first inequality follows from the triangle inequality,  the second line uses \ref{derivative_bound}, and the third inequality uses \ref{partial_bound}.
\end{itemize}

\section{Proximal Coordinate Descent}\label{app:proxCD}
Here we show linear convergence of randomized coordinate descent for $F(x) = f(x) + g(x)$ assuming that $F$ satisfies the proximal PL inequality, $\nabla f$ is coordinate-wise Lipschitz continuous, and $g$ is a separable convex function ($g(x) = \sum_i g_i(x_i)$). %We do not need each individual $g_i$ to be convex.

From coordinate-wise Lipschitz continuity of $\nabla f$ and separability of $g$, we have
\be
\label{coorLip}
F(x + y_i e_i) - F(x) \leq y_i \nabla_i f(x) + \frac{L }{ 2} y_i^2 + g_i( x_i + y_i) - g(x_i).
\ee
Given a coordinate $i$ the coordinate descent step chooses $y_i$ to minimize this upper bound on the improvement in $F$,
\[ y_i = \argmin{t_i} \left\{ t_i \nabla_i f(x) + \frac{L }{ 2} t_i^2 + g_i( x_i + t_i) - g(x_i). \right\}\]
We next  use an argument similar to~\citet{richtarikTakac} to relate the expected improvement (with random selection of the coordinates) to the function $\mathcal{D}_g$,
\begin{align*}
 \mathbb{E} \left\{  \min_{t_i}  t_i \nabla_i f(x) + \frac{L }{ 2} t_i^2 + g_i( x_i + t_i) - g_i(x_i) \right\}
 & =  \frac{1 }{ n}\sum_i   \min_{t_i}  t_i \nabla_i f(x) + \frac{L }{ 2} t_i^2 + g_i( x_i + t_i) - g_i(x_i) \nonumber \\
 & =  \frac{1 }{ n} \min_{t_1, \cdots, t_n} \sum_i  t_i \nabla_i f(x) + \frac{L }{ 2} t_i^2 + g_i( x_i + t_i) - g_i(x_i) \nonumber \\
 & =  \frac{1 }{ n} \min_{y \equiv x + (t_1,\cdots, t_n)} \langle \nabla f(x) , y  - x\rangle + \frac{L }{ 2} || y-x||^2 + g(y) - g(x) \nonumber \\
 & =  -\frac{1 }{ 2Ln} \mathcal{D}_g(L, x).
\end{align*}
(Note that separability allows us to exchange the summation and minimization operators.) 
By using this and taking the expectation of~\eqref{coorLip} we get 
\be
\mathbb{E} \left [ F(x_{k+1}) \right ]  \leq  F(x_k)  -\frac{1 }{ 2Ln} \mathcal{D}_g(L, x).
\ee
Subtracting $F^*$ from both sides and applying the proximal-PL inequality yields a linear convergence rate of $\left ( 1 - \frac{\mu }{ nL} \right )$.

%\newpage
\section*{References}
\vspace{-2em}
\renewcommand{\bibname}{}
\bibliographystyle{abbrvnat}
\bibliography{bib}

\begin{thebibliography}{57}
\providecommand{\natexlab}[1]{#1}
\providecommand{\url}[1]{\texttt{#1}}
\expandafter\ifx\csname urlstyle\endcsname\relax
  \providecommand{\doi}[1]{doi: #1}\else
  \providecommand{\doi}{doi: \begingroup \urlstyle{rm}\Url}\fi

\bibitem[Agarwal et~al.(2012)Agarwal, Negahban, and Wainwright]{agarwal}
A.~Agarwal, S.~N. Negahban, and M.~J. Wainwright.
\newblock Fast global convergence rates of gradient methods for
  high-dimensional statistical recovery.
\newblock \emph{Ann. Statist.}, pages 2452--2482, 2012.

\bibitem[Anitescu(2000)]{anitescu2000}
M.~Anitescu.
\newblock Degenerate nonlinear programming with a quadratic growth condition.
\newblock \emph{SIAM J. Optim.}, pages 1116--1135, 2000.

\bibitem[Attouch and Bolte(2009)]{AB}
H.~Attouch and J.~Bolte.
\newblock On the convergence of the proximal algorithm for nonsmooth functions
  involving analytic features.
\newblock \emph{Math. Program., Ser. B}, pages 5--16, 2009.

\bibitem[Attouch et~al.(2013)Attouch, Bolte, and
  Svaiter]{attouch2013convergence}
H.~Attouch, J.~Bolte, and B.~F. Svaiter.
\newblock Convergence of descent methods for semi-algebraic and tame problems:
  proximal algorithms, forward--backward splitting, and regularized
  gauss--seidel methods.
\newblock \emph{Mathematical Programming}, 137\penalty0 (1-2):\penalty0
  91--129, 2013.

\bibitem[Bach and Moulines(2013)]{bachMoulines2013}
F.~Bach and E.~Moulines.
\newblock Non-strongly-convex smooth stochastic approximation with convergence
  rate ${O}(1/n)$.
\newblock \emph{NIPS}, pages 773--791, 2013.

\bibitem[Beck and Teboulle(2009)]{beck2009fast}
A.~Beck and M.~Teboulle.
\newblock A fast iterative shrinkage-thresholding algorithm for linear inverse
  problems.
\newblock \emph{SIAM J. Imaging Sci.}, 2\penalty0 (1):\penalty0 183--202, 2009.

\bibitem[Bertsekas et~al.(2003)Bertsekas, Nedic, and
  Ozdaglar]{bertsekas2003convex}
D.~Bertsekas, A.~Nedic, and A.~Ozdaglar.
\newblock \emph{{Convex Analysis and Optimization}}.
\newblock Athena Scientific, 2003.

\bibitem[Bolte et~al.(2008)Bolte, Daniilidis, Ley, and
  Mazet]{bolte2008characterization}
J.~Bolte, A.~Daniilidis, O.~Ley, and L.~Mazet.
\newblock Characterization of lojasiewicz inequalities and applications.
\newblock \emph{arXiv preprint arXiv:0802.0826}, 2008.

\bibitem[Bolte et~al.(2010)Bolte, Daniilidis, Ley, and
  Mazet]{bolte2010characterizations}
J.~Bolte, A.~Daniilidis, O.~Ley, and L.~Mazet.
\newblock Characterizations of {\l}ojasiewicz inequalities: subgradient flows,
  talweg, convexity.
\newblock \emph{Transactions of the American Mathematical Society},
  362\penalty0 (6):\penalty0 3319--3363, 2010.

\bibitem[Bolte et~al.(2015)Bolte, Nguyen, Peypouquet, and Suter]{Bolte2015}
J.~Bolte, T.~P. Nguyen, J.~Peypouquet, and B.~W. Suter.
\newblock From error bounds to the complexity of first-order descent methods
  for convex functions.
\newblock \emph{arXiv:1510.08234}, 2015.

\bibitem[Craven and Glover(1985)]{craven1985}
B.~D. Craven and B.~M. Glover.
\newblock Invex functions and duality.
\newblock \emph{J. Austral. Math. Soc. (Series A)}, pages 1--20, 1985.

\bibitem[Dinuzzo et~al.(2011)Dinuzzo, Ong, Gehler, and Pillonetto]{dinuzzo2011}
F.~Dinuzzo, C.~S. Ong, P.~Gehler, and G.~Pillonetto.
\newblock Learning output kernels with block coordinate descent.
\newblock \emph{ICML}, pages 49--56, 2011.

\bibitem[Drusvyatskiy and Lewis(2016)]{drusvyatskiy2016error}
D.~Drusvyatskiy and A.~S. Lewis.
\newblock Error bounds, quadratic growth, and linear convergence of proximal
  methods.
\newblock \emph{arXiv preprint arXiv:1602.06661}, 2016.

\bibitem[Garber and Hazan(2015{\natexlab{a}})]{garber2015}
D.~Garber and E.~Hazan.
\newblock Faster rates for the {F}rank-{W}olfe method over strongly-convex
  sets.
\newblock \emph{ICML}, pages 541--549, 2015{\natexlab{a}}.

\bibitem[Garber and Hazan(2015{\natexlab{b}})]{garber2015b}
D.~Garber and E.~Hazan.
\newblock Faster and simple {PCA} via convex optimization.
\newblock \emph{arXiv:1509.05647v4}, 2015{\natexlab{b}}.

\bibitem[Gong and Ye(2014)]{gong2014linear}
P.~Gong and J.~Ye.
\newblock Linear convergence of variance-reduced stochastic gradient without
  strong convexity.
\newblock \emph{arXiv:1406.1102}, 2014.

\bibitem[Gu et~al.(2013)Gu, Lim, and Wu]{gu}
M.~Gu, L.-H. Lim, and C.~J. Wu.
\newblock Par{N}es: A rapidly convergent algorithm for accurate recovery of
  sparse and approximately sparse signals.
\newblock \emph{Numer. Algor.}, pages 321--347, 2013.

\bibitem[Hanson(1981)]{hanson1981}
M.~A. Hanson.
\newblock On sufficiency of the {K}uhn-{T}ucker conditions.
\newblock \emph{J. Math. Anal. Appl.}, pages 545--550, 1981.

\bibitem[Hoffman(1952)]{hoffman}
A.~J. Hoffman.
\newblock On approximate solutions of systems of linear inequalities.
\newblock \emph{J. Res. Nat. Bur. Stand.}, pages 263--265, 1952.

\bibitem[Hou et~al.(2013)Hou, Zhou, So, and Luo]{houNuclear}
K.~Hou, Z.~Zhou, A.~M.-C. So, and Z.-Q. Luo.
\newblock On the linear convergence of the proximal gradient method for trace
  norm regularization.
\newblock \emph{NIPS}, pages 710--718, 2013.

\bibitem[Hush et~al.(2006)Hush, Kelly, Scovel, and Steinwart]{hush}
D.~Hush, P.~Kelly, C.~Scovel, and I.~Steinwart.
\newblock {QP} algorithms with guaranteed accuracy and run time for support
  vector machines.
\newblock \emph{J. Mach. Learn. Res.}, pages 733--769, 2006.

\bibitem[Johnson and Zhang(2013)]{johnson2013stochastic}
R.~Johnson and T.~Zhang.
\newblock Accelerating stochastic gradient descent using predictive variance
  reduction.
\newblock \emph{NIPS}, pages 315--323, 2013.

\bibitem[Kadkhodaie et~al.(2014)Kadkhodaie, Sanjabi, and Luo]{kadk2014}
M.~Kadkhodaie, M.~Sanjabi, and Z.-Q. Luo.
\newblock On the linear convergence of the approximate proximal splitting
  method for non-smooth convex optimization.
\newblock \emph{arXiv:1404.5350v1}, 2014.

\bibitem[Kurdyka(1998)]{kurdyka1998gradients}
K.~Kurdyka.
\newblock On gradients of functions definable in o-minimal structures.
\newblock In \emph{Annales de l'institut Fourier}, volume~48, pages 769--784.
  Chartres: L'Institut, 1950-, 1998.

\bibitem[{Le Roux} et~al.(2012){Le Roux}, Schmidt, and
  Bach]{roux2012stochastic}
N.~{Le Roux}, M.~Schmidt, and F.~Bach.
\newblock A stochastic gradient method with an exponential convergence rate for
  finite training sets.
\newblock \emph{NIPS}, pages 2672--2680, 2012.

\bibitem[Li and Pong(2016)]{Li2016}
G.~Li and T.~K. Pong.
\newblock Calculus of the exponent of {K}urdyka-{L}ojasiewicz inequality and
  its applications to linear convergence of first-order methods.
\newblock \emph{arXiv:1602.02915v1}, 2016.

\bibitem[Liu and Wright(2015)]{Liu_osc}
J.~Liu and S.~J. Wright.
\newblock Asynchronous stochastic coordinate descent: Parallelism and
  convergence properties.
\newblock \emph{SIAM J. Optim.}, pages 351--376, 2015.

\bibitem[Liu et~al.(2014)Liu, Wright, R{\'e}, Bittorf, and Sridhar]{Liu_esc}
J.~Liu, S.~J. Wright, C.~R{\'e}, V.~Bittorf, and S.~Sridhar.
\newblock An asynchronous parallel stochastic coordinate descent algorithm.
\newblock \emph{arXiv:1311.1873v3}, 2014.

\bibitem[\L{}ojasiewicz(1963)]{loj}
S.~\L{}ojasiewicz.
\newblock A topological property of real analytic subsets (in {F}rench).
\newblock \emph{Coll. du CNRS, Les {\'e}quations aux d{\'e}riv{\'e}es
  partielles}, pages 87--89, 1963.

\bibitem[Luo and Tseng(1993)]{Luo}
Z.-Q. Luo and P.~Tseng.
\newblock Error bounds and convergence analysis of feasible descent methods: A
  general approach.
\newblock \emph{Ann. Oper. Res.}, pages 157--178, 1993.

\bibitem[Ma et~al.(2015)Ma, Tappenden, and Tak{\'a}{\v c}]{Ma2015}
C.~Ma, T.~Tappenden, and M.~Tak{\'a}{\v c}.
\newblock Linear convergence of the randomized feasible descent method under
  the weak strong convexity assumption.
\newblock \emph{arXiv:1506.02530}, 2015.

\bibitem[Meir and R{\"a}tsch(2003)]{ratsch}
R.~Meir and G.~R{\"a}tsch.
\newblock \emph{An Introduction to Boosting and Leveraging}, pages 118--183.
\newblock Springer, Heidelberg, 2003.

\bibitem[Necoara and Clipici(2016)]{Necoara2015b}
I.~Necoara and D.~Clipici.
\newblock Parallel random coordinate descent method for composite minimization:
  Convergence analysis and error bounds.
\newblock \emph{SIAM J. Optim.}, pages 197--226, 2016.

\bibitem[Necoara et~al.(2015)Necoara, Nesterov, and Glineur]{Necoara2015}
I.~Necoara, Y.~Nesterov, and F.~Glineur.
\newblock Linear convergence of first order methods for non-strongly convex
  optimization.
\newblock \emph{arXiv:1504.06298v3}, 2015.

\bibitem[Nemirovski et~al.(2009)Nemirovski, Juditsky, Lan, and
  Shapiro]{nemirovski2009robust}
A.~Nemirovski, A.~Juditsky, G.~Lan, and A.~Shapiro.
\newblock Robust stochastic approximation approach to stochastic programming.
\newblock \emph{SIAM J. Optim.}, pages 1574--1609, 2009.

\bibitem[Nesterov(2004)]{Nes04b}
Y.~Nesterov.
\newblock \emph{Introductory Lectures on Convex Optimization: A Basic Course}.
\newblock Kluwer Academic Publishers, Dordrecht, The Netherlands, 2004.

\bibitem[Nesterov(2012)]{nestrov2012}
Y.~Nesterov.
\newblock Efficiency of coordinate descent methods on huge-scale optimization
  problems.
\newblock \emph{SIAM J. Optim.}, pages 341--362, 2012.

\bibitem[Nesterov and Polyak(2006)]{nesterov2006cubic}
Y.~Nesterov and B.~T. Polyak.
\newblock Cubic regularization of newton method and its global performance.
\newblock \emph{Mathematical Programming}, 108\penalty0 (1):\penalty0 177--205,
  2006.

\bibitem[Nutini et~al.(2015)Nutini, Schmidt, Laradji, Friedlander, and
  Koepke]{julie}
J.~Nutini, M.~Schmidt, I.~H. Laradji, M.~Friedlander, and H.~Koepke.
\newblock Coordinate descent converges faster with the {G}auss-{S}outhwell rule
  than random selection.
\newblock \emph{ICML}, pages 1632--1641, 2015.

\bibitem[Polyak(1963)]{polyak}
B.~T. Polyak.
\newblock Gradient methods for minimizing functionals (in {R}ussian).
\newblock \emph{Zh. Vychisl. Mat. Mat. Fiz.}, pages 643--653, 1963.

\bibitem[Reddi et~al.(2016{\natexlab{a}})Reddi, Hefny, Sra, Poczos, and
  Smola]{reddi2}
S.~J. Reddi, A.~Hefny, S.~Sra, B.~Poczos, and A.~Smola.
\newblock Stochastic variance reduction for nonconvex optimization.
\newblock \emph{arXiv:1603.06160}, 2016{\natexlab{a}}.

\bibitem[Reddi et~al.(2016{\natexlab{b}})Reddi, Sra, Poczos, and Smola]{reddi1}
S.~J. Reddi, S.~Sra, B.~Poczos, and A.~Smola.
\newblock Fast incremental method for nonconvex optimization.
\newblock \emph{arXiv:1603.06159}, 2016{\natexlab{b}}.

\bibitem[Reddi et~al.(2016{\natexlab{c}})Reddi, Sra, Poczos, and Smola]{reddi3}
S.~J. Reddi, S.~Sra, B.~Poczos, and A.~Smola.
\newblock Fast stochastic methods for nonsmooth nonconvex optimization.
\newblock \emph{arXiv:1605.06900}, 2016{\natexlab{c}}.

\bibitem[Richt{\'a}rik and Tak{\'a}{\v{c}}(2014)]{richtarikTakac}
P.~Richt{\'a}rik and M.~Tak{\'a}{\v{c}}.
\newblock Iteration complexity of randomized block-coordinate descent methods
  for minimizing a composite function.
\newblock \emph{Math. Program., Ser. A}, pages 1--38, 2014.

\bibitem[Riedmiller and Braun(1992)]{redmiller92}
M.~Riedmiller and H.~Braun.
\newblock {RPROP} - {A} fast adaptive learning algorithm.
\newblock \emph{In: Proc. of ISCIS VII}, 1992.

\bibitem[Schmidt et~al.(2011)Schmidt, Roux, and Bach]{schmidt2011inexact}
M.~Schmidt, N.~L. Roux, and F.~Bach.
\newblock Convergence rates of inexact proximal-gradient methods for convex
  optimization.
\newblock \emph{NIPS}, pages 1458--1466, 2011.

\bibitem[Shalev-Shwartz and Zhang(2013)]{shalev-shwartz}
S.~Shalev-Shwartz and T.~Zhang.
\newblock Stochastic dual coordinate ascent methods for regularized loss
  minimization.
\newblock \emph{J. Mach. Learn. Res.}, pages 567--599, 2013.

\bibitem[Stella et~al.(2016)Stella, Themelis, and Patrinos]{stella2016forward}
L.~Stella, A.~Themelis, and P.~Patrinos.
\newblock Forward-backward quasi-newton methods for nonsmooth optimization
  problems.
\newblock \emph{arXiv preprint arXiv:1604.08096}, 2016.

\bibitem[Tseng(2010)]{tseng2010}
P.~Tseng.
\newblock Approximation accuracy, gradient methods, and error bound for
  structured convex optimization.
\newblock \emph{Math. Program., Ser. B}, pages 263--295, 2010.

\bibitem[Tseng and Yun(2009)]{tsengYun}
P.~Tseng and S.~Yun.
\newblock Block-coordinate gradient descent method for linearly constrained
  nonsmooth separable optimization.
\newblock \emph{J. Optim. Theory Appl.}, pages 513--535, 2009.

\bibitem[Wang and Lin(2014)]{wang14}
P.-W. Wang and C.-J. Lin.
\newblock Iteration complexity of feasible descent methods for convex
  optimization.
\newblock \emph{J. Mach. Learn. Res.}, pages 1523--1548, 2014.

\bibitem[Xiao and Zhang(2013)]{xiao_zhang}
L.~Xiao and T.~Zhang.
\newblock A proximal-gradient homotopy method for the sparse least-squares
  problem.
\newblock \emph{SIAM J. Optim.}, pages 1062--1091, 2013.

\bibitem[Zhang(2015)]{Zhang2015}
H.~Zhang.
\newblock The restricted strong convexity revisited: Analysis of equivalence to
  error bound and quadratic growth.
\newblock \emph{arXiv:1511.01635}, 2015.

\bibitem[Zhang(2016)]{zhang2016characterization}
H.~Zhang.
\newblock New analysis of linear convergence of gradient-type methods via
  unifying error bound conditions.
\newblock \emph{arXiv:1606.00269v3}, 2016.

\bibitem[Zhang and Yin(2013)]{zhang-yin}
H.~Zhang and W.~Yin.
\newblock Gradient methods for convex minimization: Better rates under weaker
  conditions.
\newblock \emph{arXiv:1303.4645v2}, 2013.

\bibitem[Zhang et~al.(2013)Zhang, Jiang, and Luo]{zhang2013linear}
H.~Zhang, J.~Jiang, and Z.-Q. Luo.
\newblock On the linear convergence of a proximal gradient method for a class
  of nonsmooth convex minimization problems.
\newblock \emph{J. Oper. Res. Soc. China}, 1\penalty0 (2):\penalty0 163--186,
  2013.

\bibitem[Zhou and So(2015)]{zhou2015unified}
Z.~Zhou and A.~M.-C. So.
\newblock A unified approach to error bounds for structured convex optimization
  problems.
\newblock \emph{arXiv:1512.03518}, 2015.

\end{thebibliography}

\end{document}